\documentclass{article}

\usepackage[nonatbib, preprint]{neurips_2025}

\usepackage[utf8]{inputenc} 
\usepackage[T1]{fontenc}    
\usepackage{hyperref}       
\usepackage{xcolor}         
\definecolor{darkblue}{rgb}{0, 0, 0.85}
\usepackage{soul}

\hypersetup{
    colorlinks = true,
    citecolor = darkblue,
    linkcolor = {black}
}
\usepackage{url}            
\usepackage{booktabs}       
\usepackage{amsfonts}       
\usepackage{nicefrac}       
\usepackage{microtype}      
\usepackage{bm}         
\usepackage{graphicx}		
\usepackage{wrapfig}
\usepackage{amsmath,physics} 
\usepackage{diagbox}
\usepackage{arydshln}
\usepackage{enumitem}
\usepackage{siunitx}
\usepackage{multirow}
\usepackage{tabularx}
\usepackage{hhline}		
\usepackage{xcolor}
\usepackage{mathtools} 
\usepackage{amsthm}			
\usepackage[export]{adjustbox}
\usepackage{longtable}
\usepackage{xcolor}       
\usepackage{xkcdcolors}
\usepackage{colortbl}
\usepackage{stackengine}
\usepackage{subfigure}
\usepackage[capitalize,noabbrev]{cleveref}
\usepackage[textsize=tiny]{todonotes}
\usepackage{cite} 
\usepackage{pagesel}
\stackMath

\usepackage{amsmath,amsfonts,bm, xcolor, mathtools}

\definecolor{lightgreen}{rgb}{.9,1,.9}
\definecolor{trolleygrey}{rgb}{0.5, 0.5, 0.5}
\definecolor{BrickRed}{rgb}{0.6,0,0}
\definecolor{RoyalBlue}{rgb}{0,0,0.8}
\definecolor{Tdgreen}{rgb}{0,0.4,0.7}
\definecolor{pinegreen}{rgb}{0.0, 0.47, 0.44}
\definecolor{cornellred}{rgb}{0.7, 0.11, 0.11}
\definecolor{cadmiumgreen}{rgb}{0.0, 0.42, 0.24}
\definecolor{spirodiscoball}{rgb}{0.06, 0.75, 0.99}
\definecolor{mylightblue}{rgb}{0.85, 0.90, 0.94}
\definecolor{maroon}{cmyk}{0,0.87,0.68,0.32}

\DeclarePairedDelimiterX{\infdivx}[2]{(}{)}{%
  #1\;\delimsize\|\;#2%
}
\newcommand{\infdiv}{D_{\mathrm{KL}}\infdivx}

\def\eqref#1{equation~\ref{#1}}

\def\1{\bm{1}}

\def\px{{p(\xbm)}}
\def\qx{{q(\xbm)}}
\def\py{{p(\ybm)}}
\def\qy{{q(\ybm)}}
\def\pxsig{{p_\sigma(\xbm_\sigma)}}
\def\qxsig{{q_\sigma(\xbm_\sigma)}}
\def\xsig{{\xbm_\sigma}}
\def \dx{{\dd \xbm}}

\def \dxsig{{\dd \xsig}}

\DeclareMathAlphabet{\mathsfit}{\encodingdefault}{\sfdefault}{m}{sl}
\SetMathAlphabet{\mathsfit}{bold}{\encodingdefault}{\sfdefault}{bx}{n}

\def\R{\mathbb{R}}
\def\E{\mathbb{E}}

\def\mbm{{\bm{m}}}

\def\xbm{{\bm{x}}}

\def\zbm{{\bm{z}}}
\def\ybm{{\bm{y}}}

\def\zbm{{\bm{z}}}

\def\abm{{\bm{a}}}
\def\bbm{{\bm{b}}}
\def\nbm{{\bm{n}}}
\def\bbm{{\bm{b}}}

\def\sigmabm{{\bm{\sigma}}}

\def\Xbm{{\bm{X}}}
\def\Ybm{{\bm{Y}}}
\def\Zbm{{\bm{Z}}}
\def\Hbm{{\mathbf{H}}}

\def\Ubm{{\bm{U}}}
\def\Vbm{{\bm{V}}}
\def\Sigmabm{{\bm{\Sigma}}}
\def\Pbm{{\bm{P}}}
\def\Fbm{{\bm{F}}}
\def\Ibm{{\bm{I}}}
\def\Wbm{{\bm{W}}}

\def\Ncal{{\mathcal{N}}}

\def\Dsf{{\mathsf{D}}}
\def\Dsfhat{{\mathsf{\widehat{D}}}}

\def\Tsf{{\mathsf{T}}}
\def\Tsf{{\mathsf{T}}}

\def\Dsf{{\mathsf{D}}}

\def\xbmbar{{\overline{\bm{x}}}}
\def\zbmbar{{\overline{\bm{z}}}}
\def\ybmbar{{\overline{\bm{y}}}}

\def\nbmbar{{\overline{\bm{n}}}}

\newcommand{\KL}{D_{\mathrm{KL}}}

\newcommand{\hlgreen}[1]{{\sethlcolor{lightgreen}\hl{#1}}}
\newcommand{\hlblue}[1]{{\sethlcolor{mylightblue}\hl{#1}}}

\def\defn{\,\coloneqq\,}

\theoremstyle{plain}
\newtheorem{theorem}{Theorem}
\newtheorem{theoremsup}{Theorem}

\newtheorem{corollarysup}{Corollary}
\theoremstyle{definition}

\newtheorem{assumption}{Assumption}
\theoremstyle{remark}

\newtheorem{remarksup}{Remark}

\crefname{equation}{Eq.}{Eqs.}

\usepackage[utf8]{inputenc} 
\usepackage[T1]{fontenc}    
\usepackage{hyperref}       
\usepackage{xcolor}         
\definecolor{darkblue}{rgb}{0, 0, 0.85}
\usepackage{soul}

\hypersetup{
    colorlinks = true,
    citecolor = darkblue,
    linkcolor = {black}
}
\usepackage{url}            
\usepackage{booktabs}       
\usepackage{amsfonts}       
\usepackage{nicefrac}       
\usepackage{microtype}      
\usepackage{bm}         
\usepackage{graphicx}		
\usepackage{wrapfig}
\usepackage{amsmath,physics} 
\usepackage{diagbox}
\usepackage{arydshln}
\usepackage{enumitem}
\usepackage{siunitx}
\usepackage{multirow}
\usepackage{tabularx}
\usepackage{hhline}		
\usepackage{xcolor}
\usepackage{mathtools} 
\usepackage{amsthm}			
\usepackage[export]{adjustbox}
\usepackage{longtable}
\usepackage{xcolor}       
\usepackage{xkcdcolors}
\usepackage{colortbl}
\usepackage{stackengine}
\usepackage{subfigure}
\usepackage[capitalize,noabbrev]{cleveref}
\usepackage[textsize=tiny]{todonotes}
\usepackage{cite} 
\usepackage{pagesel}
\stackMath

\title{Unsupervised Detection of Distribution Shift in\\Inverse Problems using Diffusion Models}

\vspace{5em}
\author{%
\normalsize Shirin Shoushtari \quad Edward P.~Chandler \quad Yuanhao Wang \\  
\textbf{M.~Salman Asif} \quad Ulugbek S.~Kamilov\\[0.7em]
\small \textnormal{Washington University in St.\ Louis}\\
\footnotesize \texttt{\{s.shirin, e.p.chandler, yuanhao, kamilov\}@wustl.edu}\\[0.5em]
\small \textnormal{University of California, Riverside}\\
\footnotesize \texttt{sasif@ucr.edu}
}

\begin{document}

\maketitle

\begin{abstract}
Diffusion models are widely used as priors in imaging inverse problems. However, their performance often degrades under distribution shifts between the training and test-time images. Existing methods for identifying and quantifying distribution shifts typically require access to clean test images, which are almost never available while solving inverse problems (at test time). We propose a fully \emph{unsupervised} metric for estimating distribution shifts using \emph{only} indirect (corrupted) measurements and score functions from diffusion models trained on different datasets. We theoretically show that this metric estimates the KL divergence between the training and test image distributions. Empirically, we show that our score-based metric, using only corrupted measurements, closely approximates the  KL divergence computed from clean images. Motivated by this result, we show that aligning the out-of-distribution score with the in-distribution score---using only corrupted measurements---reduces the KL divergence and leads to improved reconstruction quality across multiple inverse problems.
\end{abstract}

\section{Introduction}
\label{intro}

Standard \emph{deep learning} models typically assume that training and test data are drawn from the same distribution. However, this assumption often fails~\cite{zhang2022whyfail}, with out-of-distribution (OOD) test inputs causing significant performance degradation---specially in domains like healthcare and robotics~\cite{yang2024generalized}. Detecting and quantifying distribution shifts is thus essential for building robust models. Recent works have focused on characterizing distribution shifts~\cite{wiles2022a, peng2021wild, chen2021mandoline} and detecting OOD samples~\cite{yang2022openood} (see also reviews in~\cite{salehi2022unified, yang2024generalized}). A widely used strategy for OOD detection is based on model confidence, where softmax-based indicators---such as low maximum probability or high entropy---serve as simple yet effective proxies for detecting distribution shifts, especially in classification tasks~\cite{hendrycks2017a, liang2018enhancing}.

Diffusion  models (DMs)~\cite{ho2020denoising, song2020score} have been shown to achieve state-of-the-art performance across a wide range of tasks, including high-quality image generation~\cite{vahdat2021score, dhariwal2021diffusion, latent_diffusion, Karras2022edm, kim2022refining}, imaging inverse problems~\cite{chung2023diffusion, chung2024decomposed}, and medical imaging~\cite{chung2023solving, chung2022score, xie2022measurement, li2023descod, adib2023synthetic} (see also recent reviews~\cite{daras2024survey, kazerouni2023diffusion, croitoru2023diffusion, li2023diffusion}). These models approximate the score function of the data distribution and enable principled sampling via stochastic differential equations~\cite{song2020score}, allowing data generation from pure noise. Since diffusion models approximate the full data distribution through learned score functions, they are inherently sensitive to distribution shifts and require efficient methods for OOD detection and shift quantification. Recent work has explored this by analyzing various diffusion model-based approaches, including score consistency, sample likelihood, reconstruction error, and properties of the diffusion trajectory~\cite{heng2024out, graham2023denoising, liu2023unsupervised, livernochediffusion}.

A key limitation of existing OOD detection methods is their reliance on clean test-time images.  Moreover, recent methods that leverage diffusion models primarily focus on binary detection of OOD samples, rather than quantifying the degree of distribution shift~\cite{heng2024out, graham2023denoising, liu2023unsupervised, livernochediffusion}. In many applications, such as inverse problems, only indirect (corrupted) measurements are available, and access to clean ground-truth images is unrealistic. To address this gap, we propose the first \emph{unsupervised} metric for estimating the distribution shift between diffusion models trained on \emph{in-distribution (InD)} and OOD data, using \textit{only} corrupted measurements. Under clearly stated assumptions on the measurement operator, we prove that the proposed metric estimates the KL divergence between the underlying image distributions. Empirical results on inpainting and MRI reconstruction demonstrate that our metric, while operating solely on corrupted measurements, closely approximates the  image-domain KL divergence computed from clean images, across a variety of datasets and corruption levels. Figure~\ref{fig:metric} illustrates this behavior, comparing our proposed measurement-domain KL metric with the  KL divergence computed using clean images, under inpainting corruption with masking probabilities $p \in \{0.2, 0.5, 0.8\}$, using an InD model trained on FFHQ and OOD models trained on MetFaces, AFHQ, and Microscopy.

Domain adaptation methods aim to mitigate distribution shifts between training and test data~\cite{farahani2021brief} and are well-studied in the broader machine learning literature~\cite{zhang2022transfer, Csurka2017}. In inverse problems, however, adaptation is particularly challenging due to the unavailability of clean test-time data. Recent \emph{self-supervised} approaches have explored adapting deep learning models using only measurement-domain signals~\cite{chung2024deep, barbano2025steerable, darestani22a}, but these methods are largely heuristic and lack a theoretical justification for their effectiveness. Our metric provides a principled framework that formally connects distribution shifts to the discrepancy between score functions, evaluated directly on corrupted measurements. This connection not only quantifies the shift but also explains why adapting score functions on partial measurements can improve generalization. Empirically, we show that such adaptation reduces the estimated KL divergence and improves reconstruction quality across multiple inverse problems.

\textbf{Our contributions are:} \textbf{(1)} The first closed‑form, measurement‑domain estimate of distribution shifts that relies only on corrupted measurements. We prove that the proposed metric can equal the image‑domain KL divergence—without requiring ground‑truth images. \textbf{(2)} Empirical validation that the proposed metric closely approximates the  KL divergence across two inverse problems, including inpainting and MRI, using only corrupted measurements. \textbf{(3)} Motivated by the proposed metric, a simple adaptation approach that aligns the OOD score function with the InD data using only corrupted measurements. Our results confirm that this alignment reduces the estimated distribution shift and improves reconstruction performance in inverse problems.

\section{Background}
\label{s:bckgrnd}
\subsection{Denoising Diffusion Probabilistic Models}
Diffusion models~\cite{ho2020denoising, song2020score, Karras2022edm} are  trained to estimate the \emph{score function} of the data distribution---that is the gradient of the log-density. During training, a forward process progressively adds Gaussian noise to clean data samples $ \xbm \sim \px $ over multiple steps, while the model learns to reverse this process by denoising the corrupted samples at each step. This forward process is typically modeled as a Markov chain, $ \xbm_{\sigma_0} \rightarrow \xbm_{\sigma_1} \rightarrow \cdots \rightarrow \xbm_{\sigma_\infty} $, where $ \xbm_{\sigma_0} = \xbm $ is the clean image and noise levels $\sigma_0 < \sigma_1 < \cdots < \sigma_\infty$ increase at each step. We denote the full set of noise levels by $ \sigmabm \defn [\sigma_0, \cdots, \sigma_\infty] $, which corresponds to a time-dependent diffusion process.

\begin{figure}[t]
    \centering
    \includegraphics[width=1\textwidth]{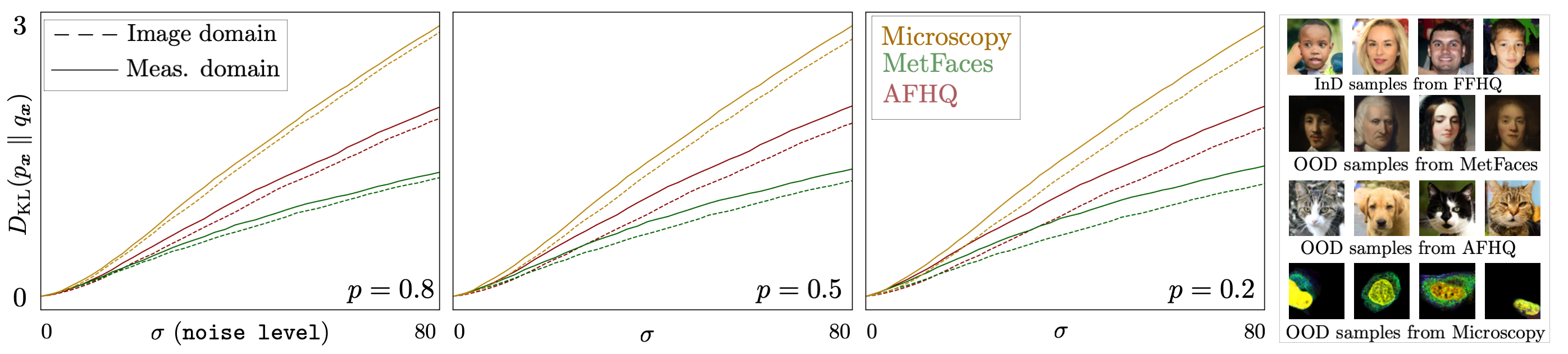}
    \caption{\small\textit{Comparison of the  distribution shift (dashed lines), computed using clean images, and our proposed measurement-domain KL metric (solid lines) between an InD model trained on FFHQ and OOD models trained on MetFaces, AFHQ, and Microscopy. Results are shown under inpainting masks with $p \in \{0.2, 0.5, 0.8\}$. The vertical axis shows $\KL$, evaluated as the integrand in~\cref{eq:mainmetric} and~\cref{eq:KLforimage} up to diffusion noise level $\sigma$. Right: Samples from InD and OOD datasets. Note how the proposed metric accurately tracks the  KL divergence, even under high-levels of corruption (smaller values of $p$).}}
    \label{fig:metric}
\end{figure}
The intermediate noisy variable $ \xbm_\sigma $ is defined using a Gaussian kernel:
\begin{equation*}
p(\xbm_\sigma|\xbm) = \Ncal(\xbm, \sigma^2 \Ibm ),
\end{equation*}
which enables direct sampling via  $ \xbm_\sigma = \xbm + \nbm $, where $ \nbm \sim \Ncal(0, \sigma^2 \Ibm) $. The marginal distribution of the noisy images, denoted $ \pxsig $, is given by:
\begin{equation}
\label{eq:ProbNoisy}
\pxsig  = \int p(\xbm_\sigma | \xbm)\px\dx = \int G_\sigma(\xbm_\sigma - \xbm)p(\xbm)\dx ,
\end{equation}
where $ G_\sigma $ denotes the Gaussian density function with standard deviation $ \sigma \geq 0 $.

Tweedie’s formula establishes a link between Gaussian denoising and score estimation~\cite{Robbins1956Empirical, Miyasawa61} by expressing the posterior mean in terms of the score of the noise-corrupted density:
\begin{equation}
\label{eq:Tweedie}
\Dsf_\sigma (\xbm_\sigma) = \E [\xbm|\xbm_\sigma]   = \xbm_\sigma + \sigma^2 \nabla \log \pxsig.
\end{equation}
This result implies that learning the Gaussian denoiser $ \Dsf_\sigma $ is equivalent to learning the score $ \nabla \log \pxsig $ of the noisy distribution, for all noise levels $ \sigma \geq 0 $. In practice, the denoiser $ \Dsf_\sigma $ is trained to minimize the mean squared error (MSE) between the clean and denoised signals:
\begin{equation}
\label{eq:MMSEdenoiser}
\text{MSE}(\Dsf_\sigma) = \E_{\xbm, \xbm_\sigma} \left [ \|\xbm - \Dsf_\sigma(\xbm_\sigma)\|_2^2\right].
\end{equation}

A diffusion model consists of a collection of MMSE denoisers across all noise levels, $\{\Dsf_\sigma: \sigma \in \sigmabm\}$, which implicitly provide access to the score functions $ \nabla \log \pxsig $ of the noise-corrupted densities. These learned score functions enable sampling from the underlying clean image distribution $ \px $  via the reverse diffusion process~\cite{vincent2011connection, Raphan10}.

\subsection{Measuring Distribution Shifts with  Clean Images using Score Functions}
We extend the framework introduced in~\cite{song2021maximum, kadkhodaiegeneralization} to derive an expression for the KL divergence between the InD $\px$ and OOD $\qx$ densities. In particular, the KL divergence can be expressed in terms of the score functions of the corresponding noise-corrupted distributions as
\begin{equation}
\label{eq:KLforimage}
  \infdiv{p(\xbm)}{q(\xbm)} = \int_0^\infty \E_{\xbm\sim p(\xbm)} \left [\|\nabla_\xsig \log \pxsig- \nabla_\xsig \log \qxsig\|_2^2  \right]~\sigma~ \dd\sigma.
\end{equation}
Here, $\pxsig$ and $\qxsig$ denote the noise-corrupted distributions of $\px$ and $\qx$, respectively, at noise level $\sigma$. The score function $\nabla_\xsig \log \pxsig$ can be estimated using the Tweedie’s formula, which relates it to the posterior mean $\E[\xbm | \xsig]$ according to~\cref{eq:Tweedie}. This posterior mean, in turn, can be approximated by training MMSE denoisers via the loss in~\cref{eq:MMSEdenoiser}.

In practice, diffusion models are trained as denoisers across a range of noise levels to approximate the score functions of the corresponding noise-corrupted data distributions.
Thus, the KL divergence in~\cref{eq:KLforimage} can be estimated when two diffusion models are available: one trained on InD samples from $\px$, and another on ODD  samples from $\qx$. 

When using diffusion models to estimate KL divergence, it is assumed that both the InD and OOD models have accurately learned the score functions of their respective data distributions. The discrepancy between their learned Gaussian denoisers at each noise level reflects the extent of the distribution shift. Leveraging the connection between the conditional mean estimator provided by the deep MMSE denoiser and the score function from~\cref{eq:Tweedie}, we obtain a tractable metric for measuring distribution shift in image domain (see Appendix~\ref{app:proofKLDiv} for the proof, as well as~\cite{song2021maximum, kadkhodaiegeneralization} for additional discussion). Notably, the resulting metric corresponds to the integrated denoising gap between the InD and OOD diffusion models across all noise levels. 

The KL divergence formulation in~\cref{eq:KLforimage} quantifies the shift between the InD density $p(\xbm)$ and the OOD density $q(\xbm)$ only when clean InD images are available. To cover the more realistic setting in which we possess only corrupted measurements, we introduce an \emph{unsupervised} metric that estimates the same distribution shift directly from those measurements.

\section{Distribution Shift in Measurement Domain} \label{s:methods}

Clean images required for the KL divergence in~\cref{eq:KLforimage} are unavailable in many inverse problems. We therefore derive a measurement-domain KL estimator that quantifies distribution shift directly from the observed measurements and pretrained diffusion models.

\subsection{Problem Formulation}
We consider a set of measurement operators randomly drawn from the distribution $p(\Hbm)$. For a given $\Hbm \in \R^{m \times n}$, the measurement vector $\ybm \in \R^m$ is related to the underlying signal $\xbm \in \R^n$ via
\begin{equation}\label{eq:measurement_model}
\ybm = \Hbm \xbm + \zbm,
\end{equation}
where $\zbm \sim \Ncal(0, \sigma_{\zbm}^2 \Ibm)$ denotes the measurement noise. We assume that $\xbm$, $\zbm$, and $\Hbm$ are independently drawn from their respective distributions for each instance of the problem.

To simplify our analysis, we consider the singular value decomposition (SVD) to the measurement operator  $\Hbm$~\cite{kawar2022denoising, kawargsure}.  This decomposition facilitates a transformation that decouples the measurement process and allows the KL divergence---originally defined in the image domain---to be re-expressed in the measurement domain. 
We write the SVD of $\Hbm$ as  
\begin{equation}\label{eq:SVDofH}
    \Hbm = \Ubm \Sigmabm \Vbm^\Tsf, 
\end{equation}
where $\Ubm \in \R^{m\times m}$ and $\Vbm \in \R^{n\times n}$ are orthogonal matrices, and $\Sigmabm \in \R^{m \times n}$ is a  matrix of singular values. We define three transformed variables: $\xbmbar = \Vbm^\Tsf \xbm$, $\ybmbar = \Sigmabm^\dagger \Ubm^\Tsf \ybm$, and $\zbmbar = \Sigmabm^\dagger \Ubm^\Tsf \zbm$. Substitution of these variables into the original measurement model in~\cref{eq:measurement_model}, leads to relationship
\begin{equation}\label{eq:svdmeasurement}
    \ybmbar = \Pbm \xbmbar + \zbmbar, 
\end{equation}
where $\Pbm = \Sigmabm^\dagger \Sigmabm$ is a diagonal projection matrix with entries in $\{0, 1\}$, and  $\zbmbar \sim \Ncal(0, \sigma_\zbm^2 \Sigmabm^\dagger {\Sigmabm^\dagger}^\Tsf )$ represents anisotropic uncorrelated Gaussian noise.

In the noiseless setting, we can rewrite~\cref{eq:svdmeasurement} as $\ybmbar = \Pbm \xbmbar$. For every noise level $\sigma$ in the noise schedule vector $\sigmabm$, we consider a noisy version of the SVD observations
\begin{equation}\label{eq:ybarsigma}
    \ybmbar_\sigma = \Pbm \xbmbar_\sigma = \Pbm \xbmbar + \nbmbar = \ybmbar +\nbmbar, \qquad \text{where} \qquad \nbmbar = \Pbm \nbm \sim \Ncal(0, \sigma^2 \Pbm),
\end{equation}
where $\nbm \sim \Ncal(0, \sigma^2 \Ibm)$ and $\Pbm$ is an orthogonal projection. Note that $\zbm$ refers to measurement noise in inverse problems, while $\nbm$ denotes noise added in the diffusion process. 

\subsection{Theoretical Results}\label{ss:theory}

We now present our main theoretical result for measuring the distribution shift between the InD prior $\px$ and OOD prior $\qx$. Theorem~\ref{thm:thm1} below presents results for noiseless measurements for imaging systems modeled as $\ybm = \Hbm \xbm$. We extend this result to the noisy case in Theorem~\ref{thm:thm2} of Appendix~\ref{app:proofthm2}.

We require the following assumptions to establish our theoretical results.
\begin{assumption}\label{as:one}
The range of the measurement operators $\Hbm \sim p(\Hbm)$, used across experiments collectively spans the signal space $\R^n$. 
\end{assumption}

The assumption enforces that, on average, the measurement operators collectively observe every signal direction—formally $\E[\Pbm]$ is full‑rank on the relevant subspace. This assumption is commonly adopted in self-supervised inverse problems~\cite{kawargsure, aggarwal2022ensure}.
\begin{assumption}\label{as:two}
    The set of measurement operators $\Hbm$ share a common right-singular matrix. Each $\Hbm$ has the form
    $\Hbm = \Ubm \Sigmabm \Vbm^\Tsf$ with a fixed matrix $\Vbm \in \mathbb{R}^{n \times n} $, for all $\Hbm \sim p(\Hbm)$.
\end{assumption}
The shared‑subspace assumption guarantees a common latent basis for all $\Hbm$, permitting direct comparison of score functions and denoising outputs across operators. This condition naturally holds in many inverse problems, such as subsampled Fourier imaging and inpainting~\cite{kawargsure, tachella2023sensing, candes2006stable}. We can now present our main result.

\begin{theorem}\label{thm:thm1}
Let $\ybmbar_\sigma = \Pbm \xbmbar + \nbmbar$ denote the noisy projected measurements at noise level $\sigma$ according to~\cref{eq:ybarsigma}. Then, the KL divergence between the InD density $\px$ and the OOD density $\qx$ can be expressed as
\begin{align}\label{eq:mainmetric}
    \infdiv{\px}{\qx} =  \int_0^\infty \E \big[\|\Wbm(\nabla \log p_\sigma(\Vbm\ybmbar_\sigma) - \nabla \log q_\sigma(\Vbm\ybmbar_\sigma))\|_2^2 \big]\sigma~ \dd\sigma, 
\end{align}
where $\Wbm = \E[\Pbm]^{-3/2}$ is a diagonal scaling matrix, $\Vbm$ is the right singular vector from SVD of $\Hbm$, and expectation is taken over $\Pbm$ and $\ybmbar \sim p(\ybmbar| \Pbm)$.
\end{theorem}
Theorem~\ref{thm:thm1} shows that, given noiseless measurements $\ybm = \Hbm \xbm$, the KL divergence between  $\px$ and $\qx$ can be computed entirely using the measurement domain data.  The KL divergence is then expressed as a functional of the difference between the InD and OOD score functions evaluated at $\Vbm \ybmbar_\sigma$, where  $\Vbm \ybmbar_\sigma$ is obtained by projecting $\ybm$ onto the row space of $\Hbm$ and perturbing it with diffusion noise $\sigma$. This result enables us to quantify the distribution shift using only the observed measurements, the known forward operator, and pre-trained score functions—without requiring access to the underlying clean images.


The weighting matrix $\Wbm$ is introduced to compensate for the effect of the projection matrix $\Pbm$, ensuring that all components contribute proportionally—particularly when the likelihood of different $\Pbm$ realizations is imbalanced. The score functions for the distributions $\px$ and $\qx$ are directly accessible from the pretrained diffusion models. The accuracy of the KL approximation is directly tied to the quality of the expectation estimate, which depends on the number of example measurements $N$ used in the computation. As $N$ increases, the empirical estimate of the expectation becomes more reliable, leading to a tighter approximation of the  KL divergence. Figure~\ref{fig:toy} illustrates this relationship using a toy example with Gaussian mixture models (GMMs), where the KL divergence between InD and OOD distributions is plotted as a function of the diffusion noise level $\sigma$. The blue curve represents the KL divergence computed in the image domain, while the red curve shows the corresponding approximation in the measurement domain under inpainting corruption with probability $p$. As shown, the measurement-domain KL closely tracks its image-domain counterpart, validating the effectiveness of our proposed metric under varying levels of measurement corruption.

The proof of Theorem~\ref{thm:thm1} is provided in Appendix~\ref{app:proofthm1}. An extension to the case of noisy measurements is presented in Theorem~\ref{thm:thm2}, with the corresponding proof in Appendix~\ref{app:proofthm2}. Notably, Theorem~\ref{thm:thm1} does not require the measurement operator $\Hbm$ to be invertible. However, when $\Hbm$ is invertible, the KL divergence expression simplifies significantly, as detailed in Corollary~\ref{crl:invertible_noiseless} in Appendix~\ref{app:proofthm1}. In this special case, there is no need to consider a set of measurement operators $\Hbm \sim p(\Hbm)$.

\begin{figure}[t]
    \centering
    \includegraphics[width=1.0\textwidth]{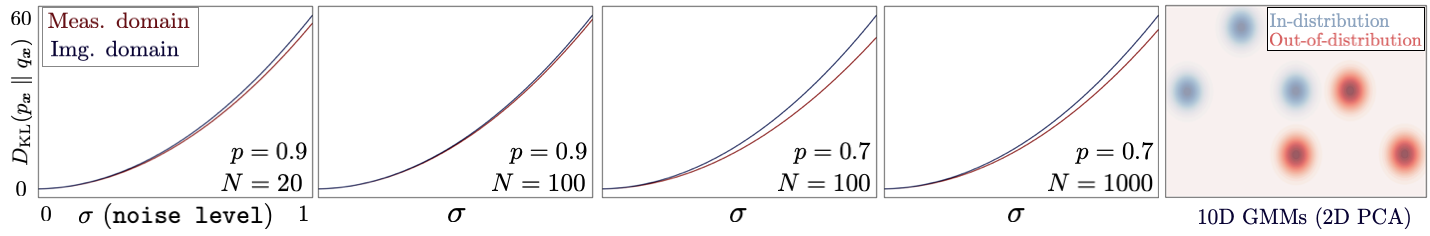}
    \caption{\small\textit{KL divergence plotted against the noise level $\sigma$ for InD and OOD Gaussian mixture models (GMMs). KL divergence computed in the image domain (blue) and measurement domain (red) under inpainting corruption with probability $p$, using $N$ InD data example. The measurement-domain KL divergence closely tracks its image-domain counterpart, and the approximation improves with increasing $N$ and $p$.}}
    \label{fig:toy}
\end{figure}

\section{Experiments}\label{s:exp}

In this section, we empirically validate the effectiveness of the proposed metric in quantifying distribution shifts. We simulate measurement operators through image inpainting and subsampled MRI, and compute the KL divergence as defined in Theorems~\ref{thm:thm1} and~\ref{thm:thm2}, under both noiseless and noisy measurement models.

Leveraging the theoretical connection between the KL divergence and the score functions, we apply an adaptation procedure to the OOD diffusion models by updating their score functions using only corrupted measurements to reduce the distribution shift. We then evaluate the impact of this adaptation on performance. Finally, we analyze how changes in the KL divergence correlate with reconstruction performance in inverse problems, showing that reducing the distribution shift leads to a better prior for diffusion model-based inference.

\subsection{Computing Distribution Shift}
\textbf{Inpainting.} To evaluate the proposed KL metric for quantifying distribution shift under inpainting corruption, we use the Flickr-Faces-HQ (FFHQ) dataset~\cite{karras2019style} as the InD data. The score function corresponding to the density $\px$ is obtained from the diffusion model trained on FFHQ. 

For the OOD densities $\qx$, we train separate diffusion models on the Animal Faces-HQ (AFHQ)\cite{choi2020starganv2}, MetFaces\cite{karras2020training}, and Microscopy (CHAMMI)~\cite{ChenCHAMMI2023} datasets. We follow the setup in~\cite{kawargsure} to simulate inpainting corruptions: images resized to $64 \times 64$ are divided into non-overlapping $4 \times 4$ patches, each of which is randomly erased with probability $p$. All diffusion models are trained using the framework in~\cite{Karras2022edm}.

We compute the KL divergence using the proposed metric in Theorem~\ref{thm:thm1}, which estimates the distribution shift directly from the corrupted measurements. This is compared to the KL divergence computed in the image domain using ~\cref{eq:KLforimage}. Assuming the diffusion models accurately capture the true score functions, the image-domain KL divergence serves as a ground truth for measuring the distribution shift.

\begin{wrapfigure}{r}{0.5\textwidth}
    \centering
    \includegraphics[width=0.48\textwidth]{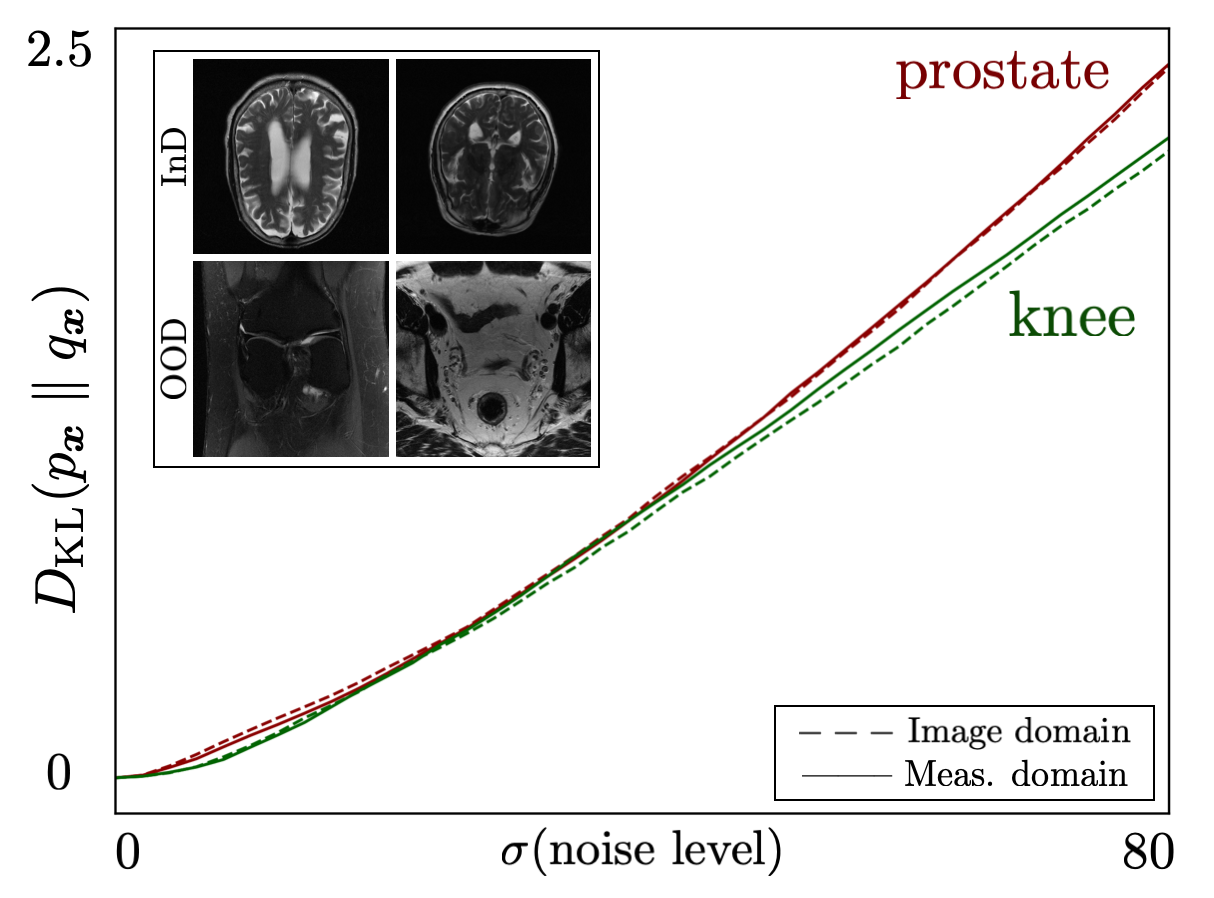}
\caption{\small\textit{Comparison of the  distribution shift (dashed lines), computed using clean images, and our proposed measurement-domain KL metric (solid lines) between an InD model trained on Brain slices and OOD models trained on Knee and Prostate slices from fastMRI dataset with acceleration rate $4$. The vertical axis shows $\KL$, evaluated as the integrand in~\cref{eq:mainmetric} and~\cref{eq:KLforimage} up to diffusion noise level $\sigma$. The proposed metric accurately tracks the  KL divergence. }}
    \label{fig:metricmri}
\end{wrapfigure}

Figure~\ref{fig:metric} presents the KL divergence between the InD and each OOD distribution, comparing the proposed measurement-domain metric with image-domain KL across inpainting rates $p \in \{0.2, 0.5, 0.8\}$. The vertical axis shows computed $\KL$ using the integrand in~\cref{eq:KLforimage} and~\cref{eq:mainmetric} up to noise level $\sigma$. Remarkably, the metric remains robust even under severe corruption ($p = 0.2$), effectively capturing the distribution gap. Although some accuracy is lost due to measurement corruption, the relative shift remains consistent—for instance, MetFaces, being visually closest to FFHQ, yields the lowest KL divergence. This demonstrates that the proposed metric reliably estimates distribution shift directly from corrupted measurements, without requiring clean images.

\textbf{Magnetic resonance imaging (MRI).}
MRI is a widely used medical imaging technique that acquires data in the frequency domain (k-space) using magnetic field gradients. To reduce scan time and patient discomfort, k-space is often under-sampled, resulting in accelerated acquisitions but yielding ill-posed inverse problems~\cite{hammernik2018learning, jalal2021robust}. 

We evaluate the proposed KL divergence metric under MRI subsampling. Brain MRI scans from the fastMRI dataset~\cite{knoll2020fastmri, zbontar2019fastmri} serve as the InD, with a diffusion model trained on center-cropped $320 \times 320$ slices to represent the score function of $\px$. Knee and prostate MRI slices from fastMRI are used for training the OOD diffusion models. To simulate accelerated MRI subsampling, we follow the protocol in~\cite{kawargsure, jalal2021robust}, applying Cartesian under-sampling masks with acceleration factors $R \in \{4, 6, 8\}$, where high-frequency components are sampled randomly.

Figure~\ref{fig:metricmri} compares KL divergence values computed in the measurement- and the image-domain. Red plots represent the distribution shift between brain (InD) and prostate (OOD) MRIs, while green plots indicate the shift between brain and knee MRIs. Dashed lines correspond to image-domain KL values, and solid lines to those computed directly from corrupted measurements. Notably, the distribution gap between brain MRI and each OOD dataset remains consistent across the KL divergence computed from clean images and MRI measurements, suggesting the reliability of the proposed KL metric for quantifying distribution shifts using only corrupted measurements. Additional experimental results are included in Appendix~\ref{ss:addexp}.

\subsection{Adaptation Effect on Distribution Shift}

Theorem~\ref{thm:thm1} introduces a new characterization of the KL divergence between InD and OOD priors by expressing it as the expected squared difference between their score functions evaluated on noisy measurements. This formulation naturally motivates adapting the denoising network using projected measurements, as reducing the score function discrepancy at those locations directly decreases the KL divergence and, consequently, the distribution shift. 

To empirically validate this, we begin with a diffusion model trained on the OOD dataset (AFHQ) and apply lightweight adaptation using projected measurements from the InD dataset (FFHQ). Specifically, we select $64$ (or $128$) clean training images from the InD distribution, obtain their corrupted measurements, and project them onto the latent basis defined by the measurement operator (i.e., compute $\Vbm \bar{\ybm}$). Figure~\ref{fig:adapted} plots the resulting KL divergence as a function of noise level $\sigma$, comparing the original AFHQ model with two adapted variants: Adapted64 (adapted using $64$ projected measurements) and Adapted128 (using $128$). Further details on the adaptation procedure are provided in Appendix~\ref{ss:adaptation}. The results show that even modest adaptation based solely on corrupted measurements significantly reduces the KL divergence.

\begin{wrapfigure}{r}{0.45\textwidth}
    \centering
    \includegraphics[width=0.45\textwidth]{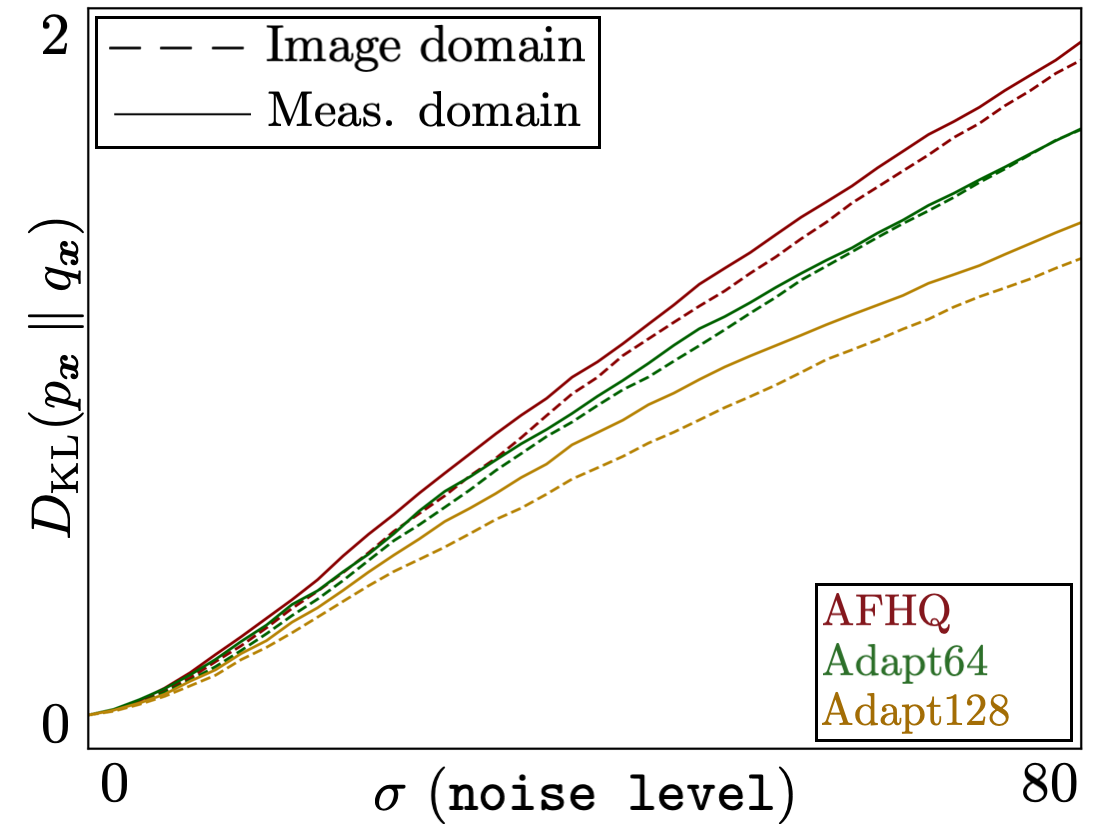}
    \caption{\small\textit{$\KL$ between FFHQ and AFHQ, as well as adapted models using 64 and 128 projected measurements, measured in the image domain (dashed) and the measurement domain (solid) for inpainting with p=0.8. Notably, adapting the  network using only projected measurements significantly reduces the distributional gap.  }}
    \label{fig:adapted}
\end{wrapfigure}

We  evaluate the impact of distribution shift—and adaptation on addressing it—on imaging inverse problems. 
We focus on solving image inpainting using DPS~\cite{chung2023diffusion} with  four models: an InD model (FFHQ), an OOD model(AFHQ), and the AFHQ model adapted using $64$ and $128$ projected measurements from FFHQ.  The adaptation procedure fine-tunes the OOD model to better approximate the score function on corrupted projected measurements from the InD, without using clean images.

Figure~\ref{fig:visual} presents a visual comparison of inpainting results on a test image from FFHQ under an inpainting mask with sampling probability $p = 0.8$ and measurement noise level $\sigma_{\zbm} = 0.01$. PSNR and LPIPS metrics are reported. As expected, the OOD model performs poorly on the InD data, while adaptation using projected measurements  improves visual quality. Table~\ref{tab:inpaint_perf} quantitatively compares DPS reconstruction quality across models and measurement settings.  Notably, the adapted models show clear gains over the unadapted OOD model, confirming that even partial measurement-based adaptation helps shrink the distribution shift.

\begin{table}[t]
\renewcommand{\arraystretch}{1.1}  
\centering
\caption{Comparison of InD, OOD, and Adapted models for image reconstruction using DPS~\cite{chung2023diffusion}, for inpainting with different inpainting mask probablity  and measurement noise. \hlblue{\textbf{Best}} and \hlgreen{\textbf{second best}} are shown. }
\begin{tabular}{l|cc|cc}
\toprule
\multirow{2}{*}{\textbf{Method}} 
& \multicolumn{2}{c}{\textbf{$ p = 0.8 \quad \sigma_{\zbm} = 0.01$}} 
& \multicolumn{2}{c}{\textbf{$ p = 0.9 \quad \sigma_{\zbm} = 0.00$}} \\
\cmidrule(lr){2-3} \cmidrule(lr){4-5}
  & PSNR $\uparrow$ & LPIPS$\downarrow$  & PSNR$\uparrow$  & LPIPS$\downarrow$  \\
\midrule
Microscopy        & $21.68$ &  $0.1466$   & $25.14$ & $0.0707$\\
MetFaces          & $25.49$ &  $0.0766$   & $29.60$ & $0.0342$\\
AFHQ              & $25.84$ &  $0.0614$   & $30.02$ & $0.0246$ \\
FFHQ              & \hlblue{$\mathbf{28.36}$} &  \hlblue{$\mathbf{0.0322}$}   & \hlblue{$\mathbf{33.24}$} & \hlblue{$\mathbf{0.0113}$} \\ \cdashline{1-5}
Adapt64 (AFHQ)    & $26.14$ &  $0.0530$   &  $30.23$    &   $0.0208$    \\
Adapt128 (AFHQ)   & \hlgreen{$\mathbf{26.52}$} &  \hlgreen{$\mathbf{0.0465}$}    &     \hlgreen{$\mathbf{30.37}$}  &   \hlgreen{$\mathbf{0.0187}$}    \\
\bottomrule
\end{tabular}
\label{tab:inpaint_perf}
\end{table} 

\begin{figure}[t]
    \centering
    \includegraphics[width=1.0\textwidth]{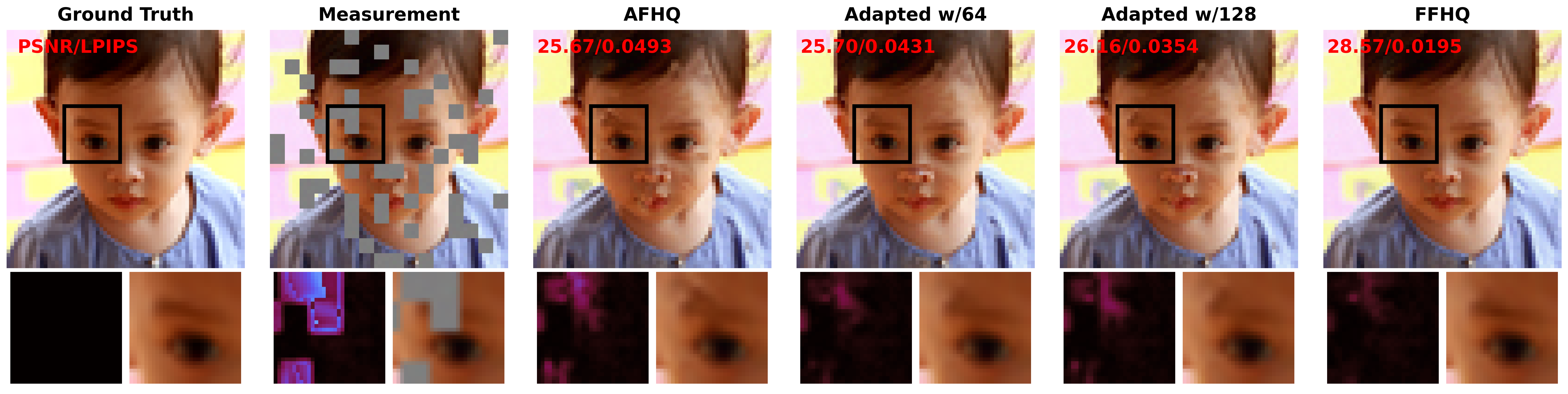 }
    \caption{\small\textit{Visual comparison of inpainting results (DPS~\cite{chung2023diffusion}) on an FFHQ image with mask rate $p = 0.8$ and measurement noise level $\sigma = 0.01$. The top row shows full reconstructions, while the bottom row displays residual maps (left) and zoomed-in regions (right). Note the performance gap between the InD and OOD models, and the improvement achieved by adapting the OOD models using only corrupted measurements.}}
    \label{fig:visual}
\end{figure}

\subsection{Ablation studies}

We study how varying the inpainting measurement probability, which controls the degree of ill-posedness, influences the accuracy of KL divergence approximation using the proposed metric. Figure~\ref{fig:diff_p} illustrates the difference between the  KL divergence computed on clean images and our metric obtained from measurements masked with varying inpainting probabilities. As expected, lower measurement corruption (i.e., higher sampling probability) leads to more accurate KL divergence estimates. However, the proposed metric remains effective in providing an approximation of the image-domain KL divergence, even under high levels of measurement corruption.
\begin{wrapfigure}{r}{0.45\textwidth}
    \centering
    \includegraphics[width=0.4\textwidth]{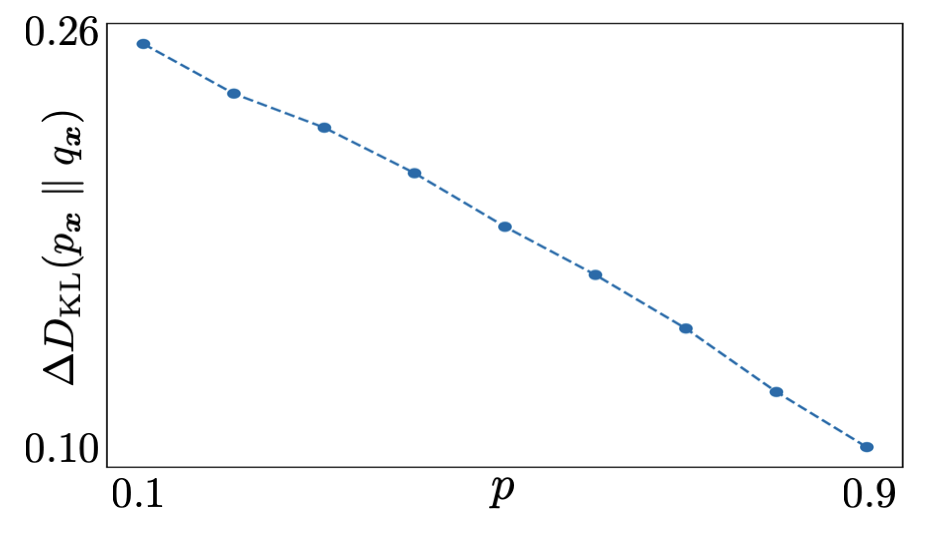}
     \caption{\small\textit{Difference between image-domain KL divergence (FFHQ vs. AFHQ) and the proposed measurement-domain approximation, plotted across varying inpainting probabilities. Smaller differences indicate better approximation; note that accuracy improves as measurement corruption decreases, while the metric remains robust even under severe corruption.}}
    \label{fig:diff_p}
\end{wrapfigure}
Table~\ref{tab:kl_ablation} presents an ablation study analyzing how the KL divergence approximation responds to measurement noise $\sigma_{\zbm}$ and number of measurement examples  $N$ from InD dataset. The results show that KL estimates remain stable even under substantial noise, supporting the robustness of the proposed metric and validating Theorem~\ref{thm:thm2}. Notably, reliable estimates are obtained with as few as 20 samples, demonstrating the metric’s effectiveness in both noisy and noiseless setting, using limited number of measurement examples from InD dataset.

\section{Conclusion}
This work presents a principled approach to unsupervised detection and quantification of distribution shift in imaging inverse problems using only corrupted measurements. By leveraging score-based diffusion models, we introduce a measurement-domain KL divergence estimator that accurately reflects the underlying distributional discrepancy without requiring access to clean test images—a major limitation in many real-world inverse problems. Our theoretical results, validated empirically on inpainting and MRI tasks, demonstrate that the proposed metric aligns well with  image-domain KL divergence and can guide adaptation strategies. Through a simple yet effective measurement-based adaptation, we show that aligning score functions reduces the estimated KL divergence and leads to improved reconstruction quality. This work not only establishes a foundation for robust evaluation of distribution shift but also offers a practical direction for mitigating its impact.

\begin{table}[t]
\centering
\renewcommand{\arraystretch}{1.1}
\setlength{\tabcolsep}{4pt}
\caption{KL divergence as a function of data examples $N$ and measurement noise level $\sigma_{\zbm}$. Note the robustness of the metric to measurement noise. Also note that limited number of corrupted measurement can approximate KL divergence.}
\begin{tabular}{!{\vrule width 0.6pt}c!{\vrule width 0.6pt}ccccc!{\vrule width 0.1pt}c!{\vrule width 0.6pt}}
\hline
\diagbox[width=4em]{\textbf{$N$}}{\textbf{$\sigma_{\zbm}$}} 
& $\mathbf{0.0}$ & $\mathbf{0.1}$ & $\mathbf{0.2}$ & $\mathbf{0.5}$ & $\mathbf{1.0}$ & $\KL$ (Img) \\
\hline
$\mathbf{20}$  & $2.098$ & $2.085$ & $2.085$ & $2.091$ & $2.114$ & $1.974$ \\
$\mathbf{40}$   & $2.070$ & $2.074$ & $2.074$ & $2.079$ & $2.102$ & $1.935$ \\
$\mathbf{80}$  & $2.063$ & $2.116$ & $2.116$ & $2.119$ & $2.140$ & $1.978$ \\
$\mathbf{120}$  & $2.073$ & $2.098$ & $2.098$ & $2.102$ & $2.124$ & $1.956$ \\
\hline
\end{tabular}
\vspace{1em}

\label{tab:kl_ablation}
\end{table}

\newpage
\section*{Limitations}
Despite its contributions, our approach has several limitations. 
The approach relies on diffusion models to estimate the data score function, which introduces inherent approximation errors that may affect the accuracy of the estimated KL divergence. However, this assumption is standard in recent works that leverage diffusion models for score-based inference~\cite{kadkhodaiegeneralization, song2019generative, song2023solving}, and such approximations have been shown to be effective in practice.
Moreover, our theoretical results rely on assumptions outlined in Section~\ref{ss:theory}, which hold for the considered measurement models (MRI and inpainting) but may not generalize to all practical settings—particularly those lacking shared right-singular vectors or full-rank projection matrices.

\section*{Impact Statement}
This paper introduces a reliable unsupervised method for quantifying distribution shifts in inverse problems using diffusion models—without requiring access to clean test data. This capability is especially valuable in high-stakes imaging applications, such as medical MRI, where acquiring ground truth is often impractical or impossible. Our approach enables practitioners to assess the robustness of pre-trained diffusion models, identify distribution mismatches, and apply lightweight adaptations to improve performance. While primarily focused on image inpainting and MRI, the methodology could be generalized to other measurement models.

{
\small

\bibliographystyle{IEEEbib}
\bibliography{references}

}


\newpage
\appendix
\section{Proof of KL Divergence Metric on Image Domain }\label{app:proofKLDiv}
 The following proof and~\cref{eq:KLforimage} results from Theorem 1 of~\cite{song2021maximum} and it is also briefly discussed in~\cite{kadkhodaiegeneralization}.  Let $\nabla_\xsig \log \pxsig$ and  $\nabla_\xsig \log \qxsig$ represent the score of InD  $\px$ and OOD  $\qx$, respectively. The distribution shift measured by KL divergence between density functions $\px$ and $\qx$ can be obtained as
\begin{equation}
   \infdiv{\px}{\qx} = \int_0^\infty \E_{\xbm\sim p(\xbm),\xsig\sim p(\xsig|\xbm)  } \left [\|\nabla_\xsig \log \pxsig- \nabla_\xsig \log \qxsig\|_2^2  \right]~\sigma~ \dd\sigma.
\end{equation}

\begin{proof}
Using the fact that $\xbm = \xbm_{\sigma_0}$, we have
\begin{align}\label{eq:DLKorig}
    &\nonumber \infdiv{\px}{\qx} \\&=  \infdiv{p(\xbm_{\sigma_0})}{q(\xbm_{\sigma_0})} -\infdiv{p(\xbm_{\sigma_\infty})}{q(\xbm_{\sigma_\infty})} + \infdiv{p(\xbm_{\sigma_\infty})}{q(\xbm_{\sigma_\infty})}\\
    & = \int_\infty^0 \frac{\partial \infdiv{p(\xbm_{\sigma})}{q(\xbm_{\sigma})}}{\partial \sigma} \dd \sigma, 
\end{align}
where in the last line, we used the fundamental theorem of calculus and  in the last line, we used the fact that $p(\xbm_{\sigma_\infty}) =  q(\xbm_{\sigma_\infty}) \approx \Ncal(0, \Ibm)$.

We calculate the derivative of $\infdiv{\pxsig}{\qxsig}$ using chain and quotient rule as:  

\begin{align} \label{eq:derSigDKL}
    \nonumber \frac{\partial \infdiv{p(\xbm_{\sigma})}{q(\xbm_{\sigma})}}{\partial \sigma} 
    & = \frac{\partial }{\partial \sigma} \int_{\R^n}\pxsig \log \frac{\pxsig}{\qxsig} \dxsig \\
    \nonumber&=  \int \frac{\partial \pxsig}{\partial \sigma} \log \frac{\pxsig}{\qxsig} \dxsig +  \int \frac{\partial \pxsig}{\partial \sigma} \dxsig -  \int \frac{\partial \qxsig}{\partial \sigma}  \frac{\pxsig}{\qxsig} \dxsig\\
    & = \int \frac{\partial \pxsig}{\partial \sigma} \log \frac{\pxsig}{\qxsig} \dxsig   -  \int \frac{\partial \qxsig}{\partial \sigma}  \frac{\pxsig}{\qxsig} \dxsig, 
\end{align}
where in the last line, we used the fact that $\int \pxsig \dxsig = 1$. 

From Fokker-Planck equation for $n$-dimensional vector $\xsig$ for the diffusion coefficient $\sigma$, we have 
\begin{equation}
    \frac{\partial \pxsig}{\partial \sigma} = \sigma \nabla^2_\xsig \pxsig. 
\end{equation}
Plugging this results in  the first term of~\cref{eq:derSigDKL} yields 
\begin{align}\label{eq:deBrujinIdenFor1}
    \nonumber \int \frac{\partial \pxsig}{\partial \sigma} \log \frac{\pxsig}{\qxsig} \dxsig & = \int  \sigma \nabla^2_\xsig \pxsig \log \frac{\pxsig}{\qxsig} \dxsig  \\
    \nonumber & = \sigma \lim_{\stackunder{\scriptscriptstyle{\abm \to \infty}} {\scriptscriptstyle{\bbm\to -\infty}}}
    \left [\nabla_\xsig \pxsig \log \frac{\pxsig}{\qxsig}\right]^\abm_\bbm \\
     \nonumber & \qquad  - \sigma \int \nabla_\xsig \pxsig^\Tsf \left[ \nabla_\xsig \log \pxsig- \nabla_\xsig \log \qxsig \right] \dxsig\\
     \nonumber& = - \sigma \int \nabla_\xsig \pxsig^\Tsf \left[ \nabla_\xsig \log \pxsig - \nabla_\xsig \log \qxsig \right] \dxsig\\
      & = - \sigma \int \nabla_\xsig \log \pxsig^\Tsf \left[ \nabla_\xsig \log \pxsig - \nabla_\xsig \log \qxsig \right] \pxsig \dxsig,
\end{align}
where we used integration by parts and the fact the the first term vanishes when both $\px$ and $\qx$ and their derivatives decays rapidly at $\pm \infty$. Note that in the last equality, we used the fact that $ \nabla_\xsig \log \pxsig \pxsig = \nabla_\xsig \pxsig $. We also have $\frac{\partial \qxsig}{\partial \sigma} = \sigma \nabla^2_\xsig \qxsig $, which yields
\begin{align}\label{eq:deBrujinIdenFor2}
     \nonumber \int \frac{\partial \qxsig}{\partial \sigma}  \frac{\pxsig}{\qxsig} \dxsig &=  \int  \sigma \nabla^2_\xsig \qxsig \frac{\pxsig}{\qxsig} \dxsig  \\
    \nonumber & = \sigma \lim_{\stackunder{\scriptscriptstyle{\abm \to \infty}} {\scriptscriptstyle{\bbm\to -\infty}}}
    \left [\nabla_\xsig \qxsig  \frac{\pxsig}{\qxsig}\right]^\abm_\bbm \\
    \nonumber & \qquad - \sigma \int \nabla_\xsig \qxsig^\Tsf \left[ \frac{\nabla_\xsig \pxsig} {\qxsig}- \frac{\nabla_\xsig \qxsig} {\qxsig} \frac{\pxsig}{\qxsig}\right] \dxsig\\
    \nonumber& = -\sigma \int \nabla_\xsig \qxsig^\Tsf \left[ \frac{\nabla_\xsig \pxsig} {\qxsig}- \frac{\nabla_\xsig \qxsig} {\qxsig} \frac{\pxsig}{\qxsig} \right] \dxsig \\
    \nonumber& = -\sigma \int \nabla_\xsig \qxsig^\Tsf \left[ \nabla_\xsig \log \pxsig- \nabla_\xsig \log \qxsig \right] \frac{\pxsig}{\qxsig} \dxsig \\
    & = -\sigma \int \nabla_\xsig \log \qxsig^\Tsf \left[ \nabla_\xsig \log \pxsig- \nabla_\xsig \log \qxsig \right] \pxsig \dxsig . 
\end{align}
Putting~\cref{eq:deBrujinIdenFor1} and~\cref{eq:deBrujinIdenFor2} in~\cref{eq:derSigDKL} establishes that 
\begin{align}
     \nonumber \frac{\partial \infdiv{p(\xbm_{\sigma})}{q(\xbm_{\sigma})}}{\partial \sigma} 
  & = -\sigma \int \pxsig \|\nabla_\xsig \log \pxsig- \nabla_\xsig \log \qxsig\|_2^2 \dxsig\\
 \nonumber &= -\sigma \E \left [\|\nabla_\xsig \log \pxsig- \nabla_\xsig \log \qxsig\|_2^2.\right]
\end{align}
Replacing this equation in~\cref{eq:DLKorig} establishes the desired result:
\begin{align}
\nonumber\infdiv{\px}{\qx} &= \int_\infty^0 \frac{\partial \infdiv{p(\xbm_{\sigma})}{q(\xbm_{\sigma})}}{\partial \sigma} \dd \sigma\\ 
& = \int^\infty_0 \E \left [\|\nabla_\xsig \log \pxsig- \nabla_\xsig \log \qxsig\|_2^2 \right] \sigma \dd \sigma.
\end{align}
\end{proof}

\section{Proof of Theorem~\ref{thm:thm1}}\label{app:proofthm1}

\begin{theoremsup}
Let $\ybmbar_\sigma = \Pbm \xbmbar + \nbmbar$ denote the noisy projected measurements at noise level $\sigma$ according to~\cref{eq:ybarsigma} and $\xbm \sim p(\xbm)$. Then, the KL divergence between the InD density $\px$ and the OOD density $\qx$ can be expressed as
\begin{align*}
    \infdiv{\px}{\qx} =  \int_0^\infty \E \big[\|\Wbm(\nabla \log p(\Vbm\ybmbar_\sigma) - \nabla \log q(\Vbm\ybmbar_\sigma))\|_2^2 \big]\sigma~ \dd\sigma, 
\end{align*}
where $\Wbm = \E[\Pbm]^{-3/2}$ is a diagonal weight matrix, and expectation is taken over $\Pbm$, $\xbm \sim p(\xbm)$, and $\ybmbar \sim p(\ybmbar|\xbm, \Pbm)$.
\end{theoremsup}

\begin{proof}
For $\xbm_\sigma = \xbm + \nbm$, where $\nbm \sim \Ncal(0, \sigma^2 \Ibm)$ is the diffusion process noise. Noting SVD for $\Hbm = \Ubm \Sigmabm \Vbm^T$, we define the transformed (right singular vector) coordinates as $\xbmbar_\sigma = \Vbm^\Tsf \xbm+ \Vbm^\Tsf \nbm = \xbmbar + \Vbm^\Tsf \nbm $. Since $\Vbm^\Tsf$ is an orthogonal matrix, the noise remains Gaussian with the same covariance, i.e., $\Vbm^\Tsf \nbm \sim \Ncal(0, \sigma^2 \Ibm)$. Applying Tweedie’s formula to the posterior $p(\xbmbar | \xbmbar_\sigma)$ yields
\begin{equation} \label{eq:scoretrasmittedUnitary}
    \nonumber \nabla \log p_\xbmbar(\xbmbar_\sigma) = \frac{\E[\xbmbar|\xbmbar_\sigma] -\xbmbar_\sigma}{\sigma^2} = \frac{\Vbm^\Tsf \E[\xbm|\xbm_\sigma] -\Vbm^\Tsf\xbm_\sigma}{\sigma^2} = \Vbm^\Tsf \frac{\E[\xbm|\xbm_\sigma] -\xbm_\sigma}{\sigma^2} = \Vbm^\Tsf \nabla \log p_\xbm(\xbm_\sigma), 
\end{equation}
where $p_\xbmbar(\xbmbar)$ is rotated distribution of $\xbm$. 
This relationship clarifies the connection between the score functions in the original space $\xbm_\sigma$ and the transformed space $\xbmbar_\sigma$. Specifically, we can write:
\begin{align}
\nonumber    \|\nabla \log p_\xbmbar(\xbmbar_\sigma) -\nabla \log q_\xbmbar(\xbmbar_\sigma)\|_2^2 &= \|  \Vbm^\Tsf \nabla \log p_\xbm(\xbm_\sigma) -  \Vbm^\Tsf \nabla \log q_\xbm(\xbm_\sigma)\|_2^2 \\
& = \|  \nabla \log p_\xbm(\xbm_\sigma) - \nabla \log q_\xbm(\xbm_\sigma)\|_2^2, 
\end{align}
where we use the fact that $\Vbm^\Tsf$ is an orthogonal matrix, and thus preserves Euclidean norms.

For $\xbmbar_\sigma = \xbmbar + \Vbm^\Tsf\nbm $ and $\Pbm= \Sigmabm^\dagger \Sigmabm$, we have 
\begin{equation}
    \ybmbar_\sigma = \Pbm \xbmbar_\sigma  = \Pbm \xbmbar + \nbmbar =  \ybmbar + \nbmbar, 
\end{equation}
where $\nbmbar = \Pbm \Vbm^\Tsf \nbm $. This implies that variance of $\nbmbar$ is $ \sigma^2\Pbm \Vbm^\Tsf\Vbm \Pbm = \sigma^2\Pbm^2 =\sigma^2\Pbm$  and $\nbmbar \sim \Ncal(0, \sigma^2 \Pbm)$ follows a Gaussian distribution with covariance $\sigma^2 \Pbm$. This reflects modified variance along the singular vector directions, due to the projection. We can write 
\begin{equation}
    p(\ybmbar_\sigma|\ybmbar) = \frac{1}{(2\pi)^{(r/2)}\sigma^r}\exp\left ( -\frac{1}{2\sigma^2} (\ybmbar_\sigma - \ybmbar)^\Tsf \Pbm (\ybmbar_\sigma - \ybmbar)  \right), 
\end{equation}
where $ r \defn \rank (\Pbm)$. The score function for $\ybmbar_\sigma$ can be written as 
\begin{align}\label{eq:scoreofmeasuremnet}
   \nonumber \nabla_{\ybmbar_\sigma} \log p(\ybmbar_\sigma) &= \frac{\nabla p(\ybmbar_\sigma)}{p(\ybmbar_\sigma)} = \frac{\nabla \int p(\ybmbar_\sigma|\ybmbar) p(\ybmbar)~\dd \ybmbar}{p(\ybmbar_\sigma)} = \frac{-1}{\sigma^2} \frac{\int \Pbm(\ybmbar_\sigma - \ybmbar)p(\ybmbar_\sigma|\ybmbar) p(\ybmbar)~\dd \ybmbar }{p(\ybmbar_\sigma)} \\
   & = \frac{\Pbm}{\sigma^2} \frac{ \int \ybmbar p(\ybmbar_\sigma|\ybmbar) p(\ybmbar)~\dd \ybmbar   - \ybmbar_\sigma \int p(\ybmbar_\sigma|\ybmbar) p(\ybmbar)~\dd \ybmbar}{p(\ybmbar_\sigma)} = \frac{\Pbm \E[\ybmbar|\ybmbar_\sigma] - \Pbm \ybmbar_\sigma}{\sigma^2}. 
\end{align}

Using the total law of expectation, we have 
\begin{align}\label{eq:expty_exptx}
 \nonumber \E[\ybmbar|\ybmbar_\sigma] &= \E_{\Pbm\sim p(\Pbm)} \big [ \E[\ybmbar|\ybmbar_\sigma, \Pbm]\big ] = \E_{\Pbm\sim p(\Pbm)} \big [ \E[\Pbm\xbmbar|\Pbm\xbmbar_\sigma, \Pbm]\big ] \\
 & = \E_{\Pbm\sim p(\Pbm)} \big [ \Pbm\E[\xbmbar|\xbmbar_\sigma]\big ] = \E_{\Pbm\sim p(\Pbm)} [\Pbm] \E \big [\xbmbar|\xbmbar_\sigma\big ], 
\end{align}
where in the second line, we used the fact that $\Pbm$ is fixed in the inner expectation and $\Pbm$ is independent of $\xbmbar$ and $\xbmbar_\sigma$. 

Now, to stablish the relation between $\nabla \log p_\xbm(\Vbm\ybmbar_\sigma)$ and $\nabla \log p_\xbm(\xbm_\sigma)$ using~\cref{eq:expty_exptx}, we have  
\begin{align}\label{eq:final_result}
    \nonumber & \E \big[\|\Wbm(\nabla \log p_{\xbm}(\Vbm\ybmbar_\sigma) - \nabla \log q_{\xbm}(\Vbm\ybmbar_\sigma))\|_2^2 \big]\\ 
    \nonumber&\overset{1}{=} \E \Big[ \text{Trace} \Big (  \Wbm \big(\nabla \log p_{\xbm}(\Vbm\ybmbar_\sigma) - \nabla \log q_{\xbm}(\Vbm\ybmbar_\sigma)\big) \big(\nabla \log p_{\xbm}(\Vbm\ybmbar_\sigma) - \nabla \log q_{\xbm}(\Vbm\ybmbar_\sigma) \big)^\Tsf \Wbm \Big )   \Big]\\
   \nonumber &  \overset{2}{=}  \E \Big[ \text{Trace} \Big (  \Wbm^2 \big(\nabla \log p_{\xbm}(\Vbm\ybmbar_\sigma) - \nabla \log q_{\xbm}(\Vbm\ybmbar_\sigma)\big) \big(\nabla \log p_{\xbm}(\Vbm\ybmbar_\sigma) - \nabla \log q_{\xbm}(\Vbm\ybmbar_\sigma) \big)^\Tsf \Big )   \Big]\\
   \nonumber &  \overset{3}{=}  \E \Big[ \text{Trace} \Big (  \frac{\Wbm^2}{\sigma^4} \big( \Pbm \E_{p_\xbm}[\Vbm\ybmbar|\Vbm\ybmbar_\sigma] - \Pbm \E_{q_\xbm}[\Vbm\ybmbar|\Vbm\ybmbar_\sigma] \big) \big(\Pbm \E_{p_\xbm}[\Vbm\ybmbar|\Vbm\ybmbar_\sigma] - \Pbm \E_{q_\xbm}[\Vbm\ybmbar|\Vbm\ybmbar_\sigma] \big)^\Tsf \Big )   \Big]\\
   \nonumber &  \overset{4}{=}  \E \Big[ \text{Trace} \Big (  \frac{\Pbm^\Tsf\Wbm^2 \Pbm}{\sigma^4} \Vbm\big(  \E_{p_\xbm}[\ybmbar|\ybmbar_\sigma] -  \E_{q_\xbm}[\ybmbar|\ybmbar_\sigma] \big) \big( \E_{p_\xbm}[\ybmbar|\ybmbar_\sigma] -  \E_{q_\xbm}[\ybmbar|\ybmbar_\sigma] \big)^\Tsf \Vbm^\Tsf\Big )   \Big]\\
   \nonumber &  \overset{5}{=}  \E \Big[ \text{Trace} \Big (  \frac{\Pbm\Wbm^2}{\sigma^4}\E[\Pbm]\Vbm\big(  \E_{p_\xbm}[\xbmbar|\xbmbar_\sigma] - \E_{q_\xbm}[\xbmbar|\xbmbar_\sigma] \big) \big( \E_{p_\xbm}[\xbmbar|\xbmbar_\sigma] -  \E_{q_\xbm}[\xbmbar|\xbmbar_\sigma] \big)^\Tsf \Vbm^\Tsf\E[\Pbm]\Big )   \Big]\\
   \nonumber &  \overset{6}{=}  \text{Trace} \Big ( \E \Big[   \frac{\Pbm \Wbm^2\E^2[\Pbm]}{\sigma^4} \big(   \Vbm\E_{p_\xbm}[\xbmbar|\xbmbar_\sigma] -  \Vbm\E_{q_\xbm}[\xbmbar|\xbmbar_\sigma] \big) \big( \Vbm\E_{p_\xbm}[\xbmbar|\xbmbar_\sigma] -  \Vbm\E_{q_\xbm}[\xbmbar|\xbmbar_\sigma] \big)^\Tsf    \Big] \Big )\\
   \nonumber &  \overset{7}{=}  \text{Trace} \Big ( \E \Big[ \E^{-1}[\Pbm] \Pbm \big( \nabla \log p_\xbm(\Vbm\xbmbar_\sigma) -\nabla \log q_\xbm(\Vbm\xbmbar_\sigma)   \big) \big( \nabla \log p_\xbm(\Vbm\xbmbar_\sigma) -\nabla \log q_\xbm(\Vbm\xbmbar_\sigma) \big)^\Tsf    \Big] \Big )\\
   \nonumber &  \overset{8}{=}  \text{Trace} \Big (  \E^{-1}[\Pbm]\E [\Pbm ] \E \Big[    \big( \nabla \log p_\xbmbar(\xbmbar_\sigma) -\nabla \log q_\xbmbar(\xbmbar_\sigma)   \big) \big( \nabla \log p_\xbmbar(\xbmbar_\sigma) -\nabla \log q_\xbmbar(\xbmbar_\sigma) \big)^\Tsf    \Big] \Big )\\
   \nonumber &  \overset{9}{=}  \text{Trace} \Big ( \E \Big[    \big( \nabla \log p_\xbmbar(\xbmbar_\sigma) -\nabla \log q_\xbmbar(\xbmbar_\sigma)   \big) \big( \nabla \log p_\xbmbar(\xbmbar_\sigma) -\nabla \log q_\xbmbar(\xbmbar_\sigma) \big)^\Tsf    \Big] \Big )\\
   \nonumber &  \overset{10}{=}   \E \Big[ \|\nabla \log p_\xbmbar(\xbmbar_\sigma) -\nabla \log q_\xbmbar(\xbmbar_\sigma)   \|_2^2   \Big]\\
    &  \overset{11}{=}   \E \Big[ \|\nabla \log p(\xbm_\sigma) -\nabla \log q(\xbm_\sigma)   \|_2^2   \Big].
   \end{align}

In equality 1, we use the identity $\|\mbm\|_2^2 = \text{Trace}(\mbm\mbm^\Tsf)$ for any vector $\mbm$. 

In equality 2, we apply the cyclic invariance of the trace operator: $\text{Trace}(\Xbm\Ybm\Zbm) =\text{Trace}(\Zbm\Xbm\Ybm) $.

In equality 3, we substitute the expression for the score function from~\cref{eq:scoreofmeasuremnet}.

In equality 4, we again use the cyclic property of the trace and the fact that $\Vbm$ is independent of $\ybmbar$ and can be taken out of the expectation. 

In equality 5, we use the fact that $\Pbm^2 = \Pbm$ for any projection matrix $\Pbm$, along with the result from~\cref{eq:expty_exptx}.

In equality 6, we again used cyclic property of the trace.

In equality 7, we again use $\Pbm^2 = \Pbm$, $\Wbm^2 \E[\Pbm] =\E^{-1}[\Pbm]$, and apply Tweedie’s formula.

In equality 8, we used the fact that the difference  between the two score functions $ \nabla \log p(\xbmbar_\sigma) -\nabla \log q(\xbmbar_\sigma)$, is independent of the projection matrix $\Pbm$, across all $\sigma$ values. We also used the fact that $p_{\xbmbar}(\xbmbar) = p_\xbm (\Vbm \xbmbar)$

In equality 11, we used the results of~\cref{eq:scoretrasmittedUnitary}. 
\end{proof}

\begin{remarksup}

    Knowing the fact that denoisers in diffusion models represent the score of noise-perturbed distribution, on can use the result of~\cref{eq:KLforimage} to obtain
\begin{align*}
    \infdiv{\px}{\qx} & = \int_0^\infty \E \Big [ \|\nabla \log p_\sigma(\xbm_\sigma) -\nabla \log q_\sigma(\xbm_\sigma)   \|_2^2  \Big]~\sigma ~ \dd\sigma \\
     & = \int_0^\infty \E \big[\|\Wbm(\nabla \log p_\sigma(\Vbm\ybmbar_\sigma) - \nabla \log q_\sigma(\Vbm\ybmbar_\sigma))\|_2^2 \big]~\sigma ~ \dd\sigma\\ 
     & =  \int_0^\infty \E \left[ \| \Wbm\left(\Dsf_\sigma\left(\Vbm\ybmbar_\sigma\right) - \Dsfhat_\sigma\left(\Vbm\ybmbar_\sigma\right)\right)\|_2^2 \right]~\sigma^{-3} ~ \dd\sigma.
\end{align*}
Here, Tweedie's formula $\sigma^2 \nabla \log p_\sigma(\Vbm\ybmbar_\sigma) = \Vbm\ybmbar_\sigma - \Dsf_\sigma(\Vbm\ybmbar_\sigma)$ and $\sigma^2 \nabla \log q_\sigma(\Vbm\ybmbar_\sigma) =\Vbm \ybmbar_\sigma - \Dsfhat_\sigma(\Vbm\ybmbar_\sigma)$ was used to acquire the last equality.
\end{remarksup}


For the case of \textbf{invertible} measurement operator, we have the following corollary. Note that here, we don't need a set of measurement operators $\Hbm$ and only one measurement operator is sufficient to derive the results. Moreover, we don't require the SVD of $\Hbm$ to obtain the resutls. 
\begin{corollarysup}\label{crl:invertible_noiseless}
Let $\ybm_\sigma = \Hbm \xbm + \nbm$ denote the noisy measurements at noise level $\sigma$ and $\xbm \sim p(\xbm)$. Then, if the measurement operator $\Hbm$ is invertible, the KL divergence between the InD density $p(\xbm)$ and the OOD density $q(\xbm)$ can be expressed as
\begin{equation*}
\infdiv{\px}{\qx} = \int_0^\infty \E  \left[ \|\nabla \log p_\sigma\left(\ybm_\sigma\right) - \nabla \log q_\sigma\left(\ybm_\sigma\right)\|_2^2 \right] \sigma \dd\sigma,
\end{equation*}
where  the expectation is taken over  $\xbm \sim p(\xbm)$, and $\ybm \sim p(\ybm|\xbm)$. 
\end{corollarysup}

\begin{proof}
From~\cref{eq:KLforimage}, similar results can be obtained for measurements $\ybm$ as 
    \begin{equation} \label{eq:DKLofmeasurement}
    \infdiv{\py}{\qy} = \int^\infty_0 \E \left [\|\nabla_{\ybm_\sigma}\log p_\sigma(\ybm_\sigma)- \nabla_{\ybm_\sigma} \log q_\sigma(\ybm_\sigma)\|_2^2  \right]~\sigma~ \dd\sigma.
    \end{equation}
   For an invertible matrix $\Hbm$, we have 
    $\py = \px .|\det(\Hbm)|^{-1}$, which yields 
    \begin{align}\label{eq:DKLxtoy}
        \nonumber \infdiv{\py}{\qy} &= \int \py \log \frac{\py}{\qy} \dd \ybm = \int \px |\det(\Hbm)|^{-1} \log \left  (\frac{\px |\det(\Hbm)|^{-1} }{\qx |\det(\Hbm)|^{-1}} \right ) \dd \ybm \\
        &= \int \px  \log \left  (\frac{\px  }{\qx} \right ) \dx = \infdiv{\px}{\qx}. 
    \end{align}
    Combining~\cref{eq:DKLofmeasurement} and~\cref{eq:DKLxtoy} establishes the desired result.
\end{proof}

\section{KL Divergence for Noisy Measurements}\label{app:proofthm2}
Here we restate  the results of Theorem~\ref{thm:thm1} for the case where measurements are corrupted by noise. In this setting, we additionally assume that the measurement noise level $\sigma_\zbm$ is known. Consider a noisy measurement $\ybm$  acquired using an imaging system according to $\ybm = \Hbm \xbm + \zbm$, where $\xbm \sim \px$ and $\zbm \sim \Ncal(0, \sigma_\zbm^2 \Ibm)$. Using SVD of $\Hbm$, we have $\ybmbar = \Pbm \xbmbar + \zbmbar $, where $\zbmbar = \Sigmabm^\dagger \Ubm^T \zbm $ as in~\cref{eq:svdmeasurement}. For every noise level $\sigma$ in noise schedule vector $\sigmabm$, we create noisy version of SVD observations as 
\begin{equation}\label{eq:ybarsigmanoisy}
    \ybmbar_\sigma = \Pbm \xbmbar_\sigma = \Pbm \xbmbar + \nbmbar +\zbmbar = \ybmbar +\nbmbar, \qquad \text{where} \qquad \nbmbar = \Pbm \nbm \sim \Ncal(0, \sigma^2 \Pbm). 
\end{equation}

\begin{theorem}\label{thm:thm2}
Let $\ybmbar_\sigma$ be obtained using~\cref{eq:ybarsigmanoisy}, then the KL divergence between density function $\px$ and $\qx$ is obtained as 
\begin{align*}
    \infdiv{\px}{\qx} =  \int_0^\infty \E \left[ \| \Wbm \left(\nabla \log p_\sigma \left(\Vbm\ybmbar_\sigma\right) - \nabla \log  q_\sigma\left(\Vbm\ybmbar_\sigma\right)\right)\|_2^2 \right]~\sigma ~ \dd\sigma, 
\end{align*}
where $\Wbm = \E[\Pbm]^{-\frac{3}{2}}$ is a diagonal weight matrix, $\ybmbar = \Sigmabm^\dagger \Ubm^\Tsf \ybm$, and expectation is taken over $\Pbm$, $\xbm \sim p(\xbm)$, and $\ybmbar \sim p(\ybmbar|\xbm, \Pbm)$. 
\end{theorem}

\begin{proof}

Similar to noiseless measurement case, we have $\xbm_\sigma = \xbm +\nbm$, $\nbm \sim \Ncal(0, \sigma^2 \Ibm)$, and  $\xbmbar_\sigma = \Vbm^\Tsf \xbm+ \Vbm^\Tsf \nbm = \xbmbar + \Vbm^\Tsf \nbm $. Using the results of~\cref{eq:scoretrasmittedUnitary}, we have 
\begin{equation}
\nonumber    \|\nabla \log p(\xbmbar_\sigma) -\nabla \log q(\xbmbar_\sigma)\|_2^2 = \|  \nabla \log p(\xbm_\sigma) - \nabla \log q(\xbm_\sigma)\|_2^2. 
\end{equation}
Applying the SVD of $\Hbm$ for  noisy measurement $\ybm$ yields
\begin{equation}
    \ybmbar = \Pbm \xbmbar + \zbmbar \qquad \text{where} \qquad \zbmbar \sim \Ncal \left( 0, \sigma_\zbm^2 \Sigmabm^\dagger {\Sigmabm^\dagger}^\Tsf\right), 
\end{equation}
By adding the noise according to diffusion model schedule to $\ybmbar$, we have 
\begin{equation}
\ybmbar_\sigma = \ybmbar + \nbmbar = \Pbm \xbmbar + \zbmbar +\nbmbar \qquad \text{where} \qquad \nbmbar \sim \Ncal \left( 0, \sigma^2 \Pbm\right). 
\end{equation}
The relation between $\ybmbar$ and $\ybmbar_\sigma$ is the same for both noisy  and noiseless measurements. Thus, by using the result from~\cref{eq:scoreofmeasuremnet} we have
\begin{equation}
   \nonumber \nabla_{\ybmbar_\sigma} \log p(\ybmbar_\sigma)  = \frac{\Pbm \E[\ybmbar|\ybmbar_\sigma] - \Pbm \ybmbar_\sigma}{\sigma^2}.
\end{equation}
To simplify $\E[\ybmbar|\ybmbar_\sigma]$, we have 
\begin{align*}
 \E[\ybmbar|\ybmbar_\sigma] &= \E_{\Pbm\sim p(\Pbm)} \big [ \E[\ybmbar|\ybmbar_\sigma, \Pbm]\big ] = \E_{\Pbm\sim p(\Pbm)} \big [ \E[\Pbm\xbmbar + \zbmbar|\ybmbar_\sigma, \Pbm]\big ] \\
 & = \E_{\Pbm\sim p(\Pbm)} \big [ \E[\Pbm\xbmbar |\ybmbar_\sigma, \Pbm] + \E[\zbmbar|\ybmbar_\sigma, \Pbm]\big ]\\
& =  \E_{\Pbm\sim p(\Pbm)} \big [ \E[\Pbm\xbmbar |\ybmbar_\sigma, \Pbm] \big ] = \E_{\Pbm\sim p(\Pbm)} \big [ \Pbm\E[\xbmbar |\ybmbar_\sigma] \big ], 
\end{align*}
where we used total law of expectation and the fact that $\E[\zbmbar|\ybmbar_\sigma, \Pbm]=0$. Since $\Pbm$ is given in the inner conditional expectation, we can factor it out of the expectation in the last line. Knowing that given $\xbmbar_\sigma$, there is no additional information about $\xbmbar$ contained in $\ybmbar_\sigma$, we have 
\begin{equation}
 \E[\ybmbar|\ybmbar_\sigma] = \E_{\Pbm\sim p(\Pbm)} \big [ \Pbm\E[\xbmbar |\ybmbar_\sigma] \big ] = \E_{\Pbm\sim p(\Pbm)} \big [ \Pbm\E[\xbmbar |\xbmbar_\sigma] \big ] = \E_{\Pbm\sim p(\Pbm)} [\Pbm] \E \big [\xbmbar|\xbmbar_\sigma\big ], 
\end{equation}
where we used the fact that $\Pbm$ is independent of $\xbmbar$ and $\xbmbar_\sigma$. Following the same logic as in~\cref{eq:final_result}, we have 
\begin{equation}
    \E \big[\|\Wbm(\nabla \log p_\sigma(\Vbm\ybmbar_\sigma) - \nabla \log q_\sigma(\Vbm\ybmbar_\sigma))\|_2^2 \big] =  \E \Big[ \|\nabla \log p(\xbmbar_\sigma) -\nabla \log q(\xbmbar_\sigma)   \|_2^2   \Big]. 
\end{equation}
The result can be obtained using the same logic from the proof of Theorem~\ref{thm:thm1} from Appendix~\ref{app:proofthm1}. 
\end{proof}

\section{Related Works}
Distribution shift between training and test data distributions is a fundamental challenge in machine learning, with direct implications for model reliability and robustness~\cite{rohan2020measuring, malinin2021shifts, zhang2022whyfail, kulinski2023towards}. Accurately measuring distribution shift is essential for understanding when models will generalize poorly, and OOD detection techniques often aim to signal such shifts by evaluating feature-based, likelihood-based, or  reconstruction-based OOD metrics~\cite{Peng2022OODdetection, fang2022learnableOODdetection, fort2021explore}. However, many existing methods rely on full access to clean samples, limiting their applicability in corrupted or measurement-limited settings~\cite{heng2024out, leBellier2024detectingood, Graham2023denoisingdiff, Gao2023semantic, liu2023unsupervised, livernoche2024on}.

To mitigate the impact of distribution shift, adaptation strategies have been developed that modify models post-training to better align with the test distribution~\cite{farahani2021brief}. In the context of diffusion models, such strategies typically focus on adjusting the generative process, modifying score functions, or fine-tuning to improve robustness against domain shifts~\cite{Kang2019contrastive, ganin2015unsupervised, You2019universal, sun2016return, shai2006analysis}. While these methods can reduce performance degradation, they often assume access to clean adaptation samples or reconstruction proxies.

In imaging inverse problems, the challenges of distribution shift, OOD detection, and adaptation are amplified by the absence of clean images at test time~\cite{gilton2021adaptation, Yismaw2024domain, shoushtari2024prior, chung2024deep, chung2023solving}. Conventional approaches to quantifying shift and adapting models are not directly applicable, as only corrupted measurements are available. This motivates the need for measurement-domain metrics and adaptation techniques that operate without requiring ground-truth reconstructions—precisely the setting we address in this work.

\section{Implementation Details}
\subsection{Inpainting}\label{ss:inpainting}
\textbf{Dataset.} We use the Flickr-Faces-HQ (FFHQ) dataset~\cite{karras2019style} as our InD data. For OOD data, we include images from the AFHQ~\cite{choi2020starganv2}, MetFaces~\cite{karras2020training}, and Microscopy (CHAMMI)~\cite{ChenCHAMMI2023} datasets. All images were resized to $64 \times 64 $for training and evaluation.

Test samples are randomly drawn from the FFHQ test set (the last $10,000$ images). For KL divergence experiments (Figure~\ref{fig:metric}), we select $20$  images (included in the supplementary materials) and process them using the inpainting measurement model. The same test set is also used for image reconstruction with the DPS algorithm.

For adaptation experiments, we sample random images from the FFHQ training set. When required by the diffusion models, data is normalized to the $[-1, 1]$ range.

\textbf{Model checkpoints.}
InD diffusion model for FFHQ and OOD AFHQ were taken from~\cite{Karras2022edm} (DDPM++ using EDM preconditioning). Similar training strategy was used  for microscopy and MetFaces diffusion models one NVIDIA A100 GPUs. All experiments regarding KL divergence were obtained using one NVIDIA RTX A6000 GPU. 

\textbf{Measurement model.} We followed the inpainting corruption setup from~\cite{kawargsure}, where the degradation operator $\Hbm$ randomly masks non-overlapping $4 \times 4$ patches across each image with probability p, independently per sample. Each $\Hbm$ is a sample-specific binary diagonal matrix that acts element-wise. As a diagonal matrix, $\Hbm$ is symmetric, idempotent, and admits the singular value decomposition $\Hbm = \Ibm \Sigmabm \Ibm^\top$, where $\Sigmabm = \Hbm$ has entries in \{0,1\}. This implies that the projection matrix $\Pbm = \Hbm^\top \Hbm = \Hbm$, and all measurement operators share the same right-singular vectors $\Vbm = \Ibm$, satisfying Assumption~\ref{as:two}. Moreover, the stochastic nature of the masking ensures that all pixels are eventually observed across different \(\Hbm\), and the union of their row spaces spans $\R^n$, satisfying Assumption~\ref{as:one}.

\subsection{FastMRI}\label{ss:mri}
\textbf{Datasets.}
We use brain MRI images from the fastMRI dataset~\cite{knoll2020fastmri, zbontar2019fastmri} as the InD data. All images are center-cropped to a resolution of $320 \times 320$ for training. The training set consists of 48,406 slices, where only slices with index greater than 5 are included. For OOD data, we extract 29,877 slices from single-coil knee MRI scans and 7,673 slices from prostate MRI scans. For evaluation, 20 images from the brain MRI validation set are used as the test set.

\textbf{Model checkpoints.}
Diffusion models for all three datasets were trained using~\cite{Karras2022edm} (DDPM++ using EDM preconditioning) using one NVIDIA A100 GPUs. All experiments regarding KL divergence were obtained using one NVIDIA RTX A6000 GPU. 

\textbf{Measurement model.} We followed the MRI measurement setup from~\cite{jalal2021robust, kawargsure} to create the corrupted data. The measurement operator $\Hbm$ performs partial Fourier sampling along the frequency (readout) axis, with an acceleration factor $R$. Specifically, $\Hbm$ retains the lowest $120/R$ frequency components and randomly selects an additional $200/R$ frequencies from the remaining spectrum, yielding a total of $320/R$ retained lines out of $320$. The operator can be expressed as $\Hbm = \Ibm \Sigmabm \Fbm$, where $\Fbm$ denotes the discrete Fourier transform and $\Sigmabm$ is a diagonal binary matrix encoding the sampling pattern. This representation serves as a valid SVD of $\Hbm$ and can be efficiently implemented via FFT. The structure of $\Hbm$ satisfies our theoretical assumptions: it is known at inference time, its right-singular vectors are $\Fbm$ and  shared across all samples  (satisfying Assumption~\ref{as:two}),and the combination of fixed low-frequency sampling with randomized high-frequency selection ensures that the union of observed frequency components across samples covers the full signal space(satisfying Assumption~\ref{as:one}).
\subsection{KL divergence experiments on GMMs}
To visualize and validate KL divergence estimation between distributions under varying noise levels and partial observations, we designed a synthetic setup using Gaussian mixture models (GMMs) in a $10$-dimensional space. Both the InD and OOD  were defined as GMMs with $K = 3$ components, each having equal weights and isotropic Gaussian covariances. The component means of the InD distribution were arranged to form a structured triangular configuration in the first two principal dimensions of $\R^{10}$: component means were placed along the x-axis with offsets of $5$ units, while alternating vertically in the y-direction to create separation. Specifically, the InD means were defined as $[0, 0, 0, \ldots], [5, 5, 0, \ldots], [10, 0, 0, \ldots]$, with the remaining eight dimensions set to zero. All InD components shared identical covariance matrices, set to the $10 \times 10$ identity matrix, yielding isotropic spreads in all directions.

To construct the OOD distribution, each InD component mean was shifted in the first two dimensions: $10$ units along the x-axis and $-5$ units along the y-axis. This resulted in OOD component centers that were clearly displaced from their InD counterparts: $[10, -5, 0, \ldots], [15, 0, 0, \ldots], [20, -5, 0, \ldots]$. Covariance matrices for the OOD components were again isotropic and identical to the InD case. This setup ensures that the only difference between InD and OOD distributions lies in their location, allowing for a clean assessment of distributional shift without confounding factors such as varying shape or spread.

We approximated the KL divergence metrics both in data and measurement domain using the corresponding formulas using a Riemann sum over $\sigma \in [0.01, 1.0]$. In a measurement-corrupted scenario, we applied random masking to the data with a given probability, zeroing out entries to simulate missing observations (e.g., similar to inpainting). We then computed the same score-based KL on the masked data and compared it to the full-data KL. Visualizations in Figure~\ref{fig:toy} include  PCA projections of the InD and OOD samples in $2D$, showing clear spatial separation in the first two dimensions. Our results show that the KL divergence computed from partially observed (masked) data closely tracks the divergence computed from clean data for various inpainting probablity $p$ and number of samples used for KL metric computing $N$. 

\subsection{Adaptation}\label{ss:adaptation}
Theorem~\ref{thm:thm1} provides the following metric for measuring distribution shift in term of KL divergence computed only on measurement data 
\begin{align*}
    \infdiv{\px}{\qx} =  \int_0^\infty \E \big[\|\Wbm(\nabla \log p_\sigma(\Vbm\ybmbar_\sigma) - \nabla \log q_\sigma(\Vbm\ybmbar_\sigma))\|_2^2 \big]\sigma~ \dd\sigma. 
\end{align*}
We update the OOD model by minimizing the following loss function on the limited corrupted projected measurements from the InD distribution 
\begin{equation}
\label{eq:lossadapt}
\text{MSE}(\Dsf_{\text{Adapted}}) = \E_{\ybm, \ybm_\sigma} \left [ \|\Wbm(\Vbm\ybmbar - \Dsf_\text{Adapted}(\Vbm\ybmbar_\sigma))\|_2^2\right].
\end{equation}

The training follows the training for diffusion models  from~\cite{Karras2022edm} (DDPM++ using EDM preconditioning), without changing the parameters (only batch number was adjusted based on the number of corrupted measurements used). For each batch, same inpainting/MRI mask was used.  Adaptation was done using one NVIDIA RTX A6000 GPU.  Data-prepartion for the adaptation follows the same procedure for calculating the KL divergence, noted in sections~\ref{ss:inpainting} and~\ref{ss:mri}. 
Adaptation is terminated when (i) the training loss fails to improve by
$\ge 0.5\%$ over the last $10\,\text{kimg}$, or (ii) $B$ kimg have been
processed ($B=500$ for 64 measurements, $B=1000$ for 128).

\subsection{Additional Experiments}\label{ss:addexp}
\subsubsection{MRI.}
Figure~\ref{fig:sup_metric_mri} extends our evaluation to the MRI measurements using different subsampling masks, comparing the  KL divergence—computed from clean brain MRI slices—with our measurement-domain KL metric, using only undersampled k-space measurements. The InD model is trained on brain MRI data, while the OOD models are trained on knee and prostate scans from the fastMRI dataset. Results are shown for acceleration rates $R \in \{4, 6, 8\}$, with the vertical axis representing the truncated KL divergence integrated up to diffusion noise level $\sigma$. As shown, the proposed metric closely follows the  KL divergence across all settings, demonstrating its robustness even under aggressive subsampling. Example slices from each dataset are shown on the right.

Figure~\ref{fig:sup_adapted_mri} demonstrates the effect of model adaptation on reducing distribution shift in the MRI setting. We plot the KL divergence between Brain and Prostate MRI slices, both before and after adapting the OOD model using only 64 projected (corrupted) measurements. Results are shown for an acceleration rate of R = 4, with KL evaluated in both the image domain (dashed) and measurement domain (solid). As shown, adaptation using only projected measurements substantially reduces the KL divergence, confirming the effectiveness of our adaptation strategy in bridging the distributional gap without requiring clean images.

\begin{figure}[t]
    \centering
    \includegraphics[width=1\textwidth]{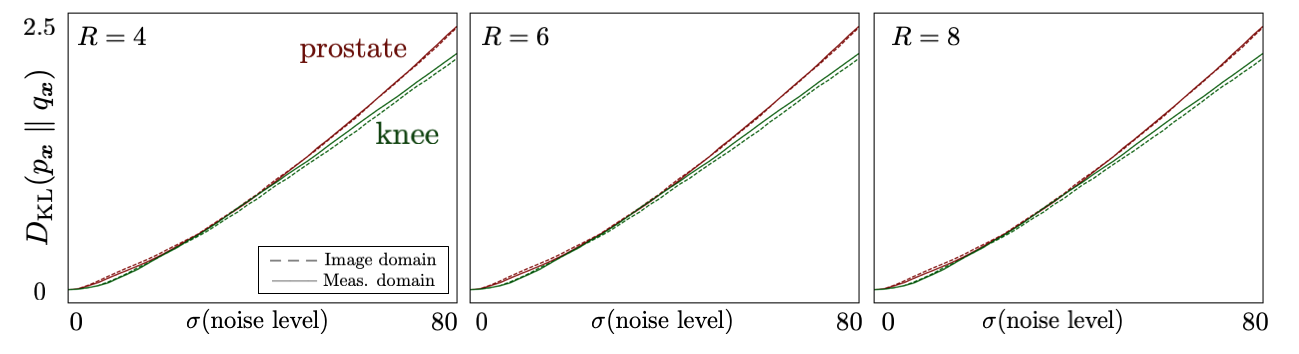}
    \caption{\small\textit{Comparison of the  distribution shift (dashed lines), computed using clean images, and our proposed measurement-domain KL metric (solid lines) between an InD model trained on Brain and OOD models trained on Knee and Prostate MRI slices from fastMRI dataset. Results are shown under MRI acceleration rates $R \in \{4,6,8\}$. The vertical axis shows $\KL$, evaluated as the integrand in~\cref{eq:mainmetric} and~\cref{eq:KLforimage} up to diffusion noise level $\sigma$. The proposed metric accurately tracks the  KL divergence, even under high-levels of corruption. Right: Samples from InD and OOD datasets.}}
    \label{fig:sup_metric_mri}
\end{figure}

\begin{figure}[t]
    \centering
    \includegraphics[width=0.45\textwidth]{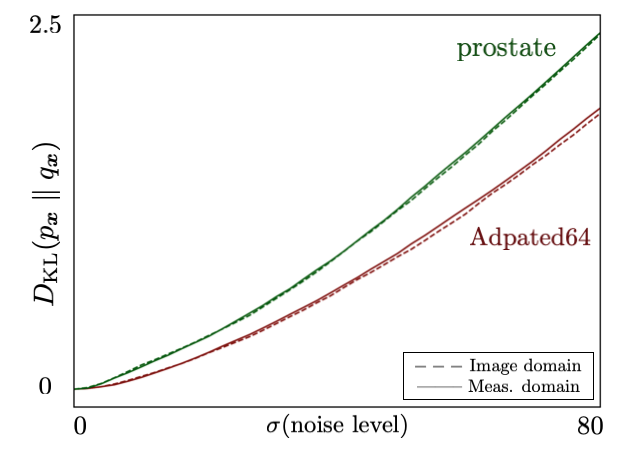}
    \caption{\small\textit{$\KL$ between Brain MRI and Prostate MRI, as well as adapted models using 64 projected measurements, measured in the image domain (dashed) and the measurement domain (solid) for subsampled MRI with acceleration rate $R=4$. Notably, adapting the network using only projected measurements significantly reduces the distributional gap.  }}
    \label{fig:sup_adapted_mri}
\end{figure}

Table~\ref{tab:sup_ab_mri} reports the KL divergence between Brain (InD) and Prostate (OOD) MRI distributions under varying acceleration rates $R \in \{4, 6, 8\}$ and measurement noise levels $\sigma_{\zbm} \in \{0.0, 0.1, 0.2\}$.  The most right column shows the  KL divergence computed in the image domain. Across all settings, the measurement-domain KL estimates remain stable and closely match the image-domain value, demonstrating the robustness of our metric to both subsampling and high levels of measurement noise.

\begin{table}[t]
\centering
\caption{KL divergence between Brain (InD) and Prostate (OOD) as a function of MRI acceleration rate $R$ and measurement noise level $\sigma_{\zbm}$. Note the robustness of the metric to measurement noise.}
\renewcommand{\arraystretch}{1.1}
\setlength{\tabcolsep}{4pt}
\begin{tabular}{!{\vrule width 0.6pt}c!{\vrule width 0.6pt}ccc!{\vrule width 0.1pt}c!{\vrule width 0.6pt}}
\hline
\diagbox[width=4em]{\textbf{$R$}}{\textbf{$\sigma_{\zbm}$}} 
& $\mathbf{0.0}$ & $\mathbf{0.1}$ & $\mathbf{0.2}$  &  $\KL$ (Img) \\
\hline
$\mathbf{4}$  & $2.51875$ & $2.53045$ & $2.53055$& $2.50662$  \\
$\mathbf{6}$   & $2.51844$ & $2.53003$ & $2.53027$ & $2.50662$ \\
$\mathbf{8}$   & $2.51821$ & $2.52980$ & $2.53002$ & $2.50662$ \\
\hline
\end{tabular}
\vspace{1em}

\label{tab:sup_ab_mri}
\end{table}

\begin{table}[t]
\renewcommand{\arraystretch}{1.1}  
\centering
\caption{Comparison of InD, OOD, and Adapted models for image reconstruction using DPS, for single-coil MRI reconstruction with for acceleration ratio $R=4$ and different measurement noise. }
\begin{tabular}{l|cc|cc}
\toprule
\multirow{2}{*}{\textbf{Method}} 
& \multicolumn{2}{c}{\textbf{$ R = 4 \quad \sigma_{\zbm} = 0.00$}} 
& \multicolumn{2}{c}{\textbf{$ R = 4  \quad \sigma_{\zbm} = 0.01$}} \\
\cmidrule(lr){2-3} \cmidrule(lr){4-5}
  & PSNR $\uparrow$ & LPIPS$\downarrow$  & PSNR$\uparrow$  & LPIPS$\downarrow$  \\
\midrule
Prostate          & $24.15$ &  $0.3223$   & $23.89$ & $0.3268$\\
Knee               & $26.51$ &  $0.2697$   & $25.90$ & $0.2774$\\
Brain              & $27.92$ &  $0.2159$   & $ 27.42$ & $0.2234$\\  \cdashline{1-5}
Adapt64 (Prostate)    & $25.17$ &  $0.3071$   & $24.80$ & $0.3089$\\
\bottomrule
\end{tabular}
\label{tab:sup_mri}
\end{table} 

\begin{figure}[t]
    \centering
    \includegraphics[width=1.0\textwidth]{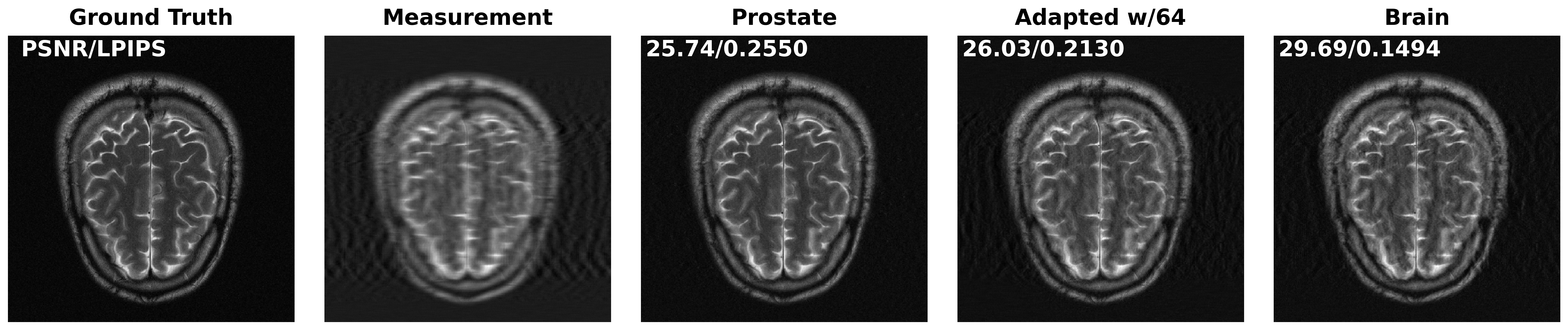}
    \includegraphics[width=1.0\textwidth]{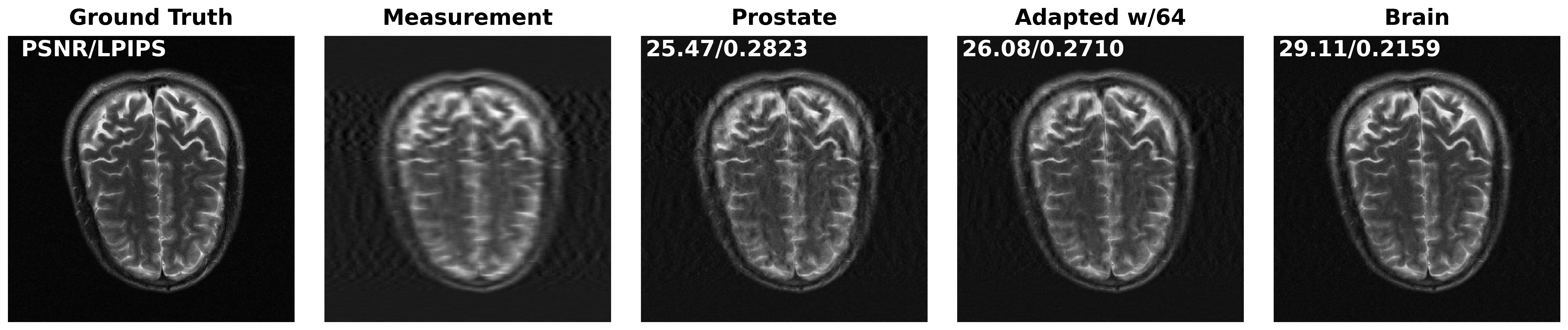}\caption{\small\textit{Visual comparison of single-coil MRI reconstruction using  DPS~\cite{chung2023diffusion} on a Brain MRI slice with acceleration ratio $R = 4$ and no measurement noise.  Note the performance gap between the InD and OOD models, and the improvement achieved by adapting the OOD models using only corrupted measurements.}}
    \label{fig:sup_visual_mri}
\end{figure}

\end{document}